\documentclass[12pt]{colt2018} 

\usepackage{color}
\usepackage{algorithm}
\usepackage{algorithmic}
\usepackage{amsmath}
\usepackage{amssymb}
\usepackage{comment}
\usepackage{graphicx}
\usepackage{bbm}
\usepackage{url}
\usepackage{xr}
\usepackage{soul}
\usepackage{listings}

\usepackage{xspace}

\usepackage{cleveref}
\usepackage{enumitem}

\crefformat{equation}{(#2#1#3)}
\crefname{lemma}{lemma}{lemmas}
\Crefname{lemma}{Lemma}{Lemmas}

\def\reals{\mathbb{R}}

\usepackage{xcolor}
\definecolor{DarkGreen}{rgb}{0.1,0.5,0.1}

\def\g{\gamma_{\K}}

\def\reals{\mathbb{R}}
\def\K{\mathcal{K}}

\def\argmin{\mathop{\arg\min}}
\def\argmax{\mathop{\arg\max}}

\def\balpha{\boldsymbol{\alpha}}

\def\alg{\text{OAlg}}

\def\BTL{\textsc{BeTheLeader}\xspace}
\def\FTL{\textsc{FollowTheLeader}\xspace}
\def\FTRL{\textsc{FollowTheRegularizedLeader}\xspace}
\def\BTRL{\textsc{BeTheRegularizedLeader}\xspace}
\def\BTPL{\textsc{BeThePerturbedLeader}\xspace}
\def\OFTRL{\textsc{Optimistic-FTRL}\xspace}
\def\BR{\textsc{BestResponse}\xspace}
\def\OFTL{\textsc{Optimistic-FTL}\xspace}

\newcommand{\norm}[1]{\left\lVert#1\right\rVert}
\newcommand{\norme}[1]{\norm{#1}}
\newcommand{\lr}[2]{\left\langle#1,#2\right\rangle}

\newcommand{\yx}[1]{y_{#1}}
\newcommand{\uregret}[1]{\textsc{Reg}^{#1}_T}
\newcommand{\regret}[1]{\balpha\textsc{-Reg}^{#1}}
\newcommand{\avgregret}[1]{\overline{\balpha\textsc{-Reg}}^{#1}}

\newcommand{\XX}{\mathcal{X}}
\newcommand{\YY}{\mathcal{Y}}
\newcommand{\pr}[1]{\left(#1\right)}

\title[Faster Rates for Convex-Concave Games]{Faster Rates for Convex-Concave Games}
\usepackage{times}

 \coltauthor{\Name{Jacob Abernethy} \Email{prof@gatech.edu}\\
 \addr Georgia Tech
 \AND
 \Name{Kevin A. Lai} \Email{kevinlai@gatech.edu}\\
 \addr Georgia Tech
 \AND
 \Name{Kfir Y. Levy} \Email{yehuda.levy@inf.ethz.ch}\\
 \addr ETH Zurich
 \AND
 \Name{Jun-Kun Wang} \Email{jimwang@gatech.edu}\\
 \addr Georgia Tech}

\begin{document}

\maketitle
\begin{abstract}
We consider the use of no-regret algorithms to compute equilibria for particular classes of convex-concave games. While standard regret bounds would lead to convergence rates on the order of $O(T^{-1/2})$, recent work \citep{RS13,SALS15} has established $O(1/T)$ rates by taking advantage of a particular class of optimistic prediction algorithms. In this work we go further, showing that for a particular class of games one achieves a $O(1/T^2)$ rate, and we show how this applies to the Frank-Wolfe method and recovers a similar bound \citep{D15}. We also show that such no-regret techniques can even achieve a linear rate, $O(\exp(-T))$, for equilibrium computation under additional curvature assumptions.
\end{abstract}


\section{Introduction}

A large number of core problems in statistics, optimization, and machine learning, can be framed as the solution of a two-player zero-sum game. Linear programs, for example, can be viewed as a competition between a feasibility player, who selects a point in $\reals^n$, and a constraint player that aims to check for feasibility violations \citep{Adler2013}. Boosting \citep{freund1999adaptive} can be viewed as the competition between an agent that selects hard distributions and a weak learning oracle that aims to overcome such challenges \citep{freund1996game}. The hugely popular technique of Generative Adversarial Networks (GANs) \citep{goodfellow2014generative}, which produce implicit generative models from unlabelled data, has been framed in terms of a repeated game, with a distribution player aiming to produce realistic samples and a discriminative player that seeks to distinguish real from fake. 

While many vanilla supervised learning problems reduce to finding the minimum of an objective function $f(\cdot)$ over some constraint set, tasks that require the search for a \emph{saddle point}---that is, min-max solution of some convex-concave payoff function $g(\cdot,\cdot)$---don't easily lend themselves to standard optimization protocols such as gradient descent, Newton's method, etc. It is not clear, for example, whether successive iterates should even increase or decrease the payoff $g$. This issue has been noticed in the training of GANs, for example, where the standard update method is a simultaneous gradient descent procedure, and many practitioners have raised concerns about cycling.

On the other hand, what has emerged as a very popular and widely-used trick is the following: simulate a pair of online learning algorithms, each competing in the game with the objective of \emph{minimizing regret}, and return the time-averaged sequence of actions taken by the players as an approximate solution. The method of applying no-regret learning strategies to find equilibria in zero-sum games was explored in \citet{freund1999adaptive}, yet the idea goes back at least as far as work by \citet{blackwell1956analog} and \citet{hannan1957approximation}. This methodology has several major benefits, which include the following. First, this framework ``decouples'' the optimization into two parallel routines that have very little communication overhead. Second, the use of no-regret learning is ideal in this scenario, as most of the guarantees for such algorithms are robust to even adversarial environments. Third, one is able to bound the approximation error of the returned saddle point simply in terms of the total regret of the two players. Finally, several surprising recent results have suggested that this parallel online learning methodology leads to \emph{even stronger guarantees} than what the na\"{i}ve theory would tell you. In short, whereas the typical no-regret analysis would lead to an approximation error of $O(T^{-1/2})$ after $T$ iterations, the use of \emph{optimistic learning} strategies \citep{CJ12} can be shown to guarantee $O(T^{-1})$ convergence; this technique was developed by \citet{RK13} and further expanded by \citet{SALS15}.

In this work we go further, showing that \emph{even faster rates} are achievable for some specific cases of saddle-point problems. In particular:
\begin{enumerate}[itemsep=0mm]
	\item \citet{AW17} observed that the optimization method known as \emph{Frank-Wolfe} is simply an instance of the above no-regret framework for solving a particular convex-concave game, leading to a rate of $O(T^{-1})$. In this work we further analyze the Frank-Wolfe game, and
	show that when the objective function and constraint set have additional structure, and both algorithms use optimistic learning procedures, then we can achieve a rate of $O(T^{-2})$. This generalizes a result of \citet{D15} who proved a similar convergence rate for Frank-Wolfe.
  \item Additionally, we show that when the game payoff function is suitably curved in both inputs---i.e. it is strongly-convex-concave and smooth---then we can use no-regret dynamics to achieve a \emph{linear rate}, with the error decaying as $O(\exp(-T))$. Applying our technique to the  {Frank-Wolfe} game we are able 
to recover the linear rate  results of \citet{LP66,DR70} and \cite{D79}.

\end{enumerate}
A notable aspect of our work is the combination of several key algorithmic techniques. First, our  Frank-Wolfe result relies on regularization using the \emph{squared gauge function}, allowing the learner to need only a single linear optimization call on each round. Second, we introduce a notion of weighted regret minimization, and our rates depend on the careful selection of the weight schedule as well as a careful analysis of what has been called \emph{Optimistic FollowTheRegularizedLeader}. Third, our linear convergence rate leans on a trick developed recently by \citet{L17} that generates an adaptive weighting scheme based on the norm of the observed gradients.

\subsection{Preliminaries} \label{Pre}

We first provide some definitions that are used in this paper. Let $f:\XX \mapsto \reals$ be some function.

\begin{definition}
	A vector $w$ is a subgradient of $f$ at $v$ for any $u \in \text{dom} f$, $f(u) \ge f(v) + \lr{w}{u-v}$. 
\end{definition}

\begin{definition}
$f$ is $L$-smooth w.r.t. a norm $\| \cdot \|$ if $f$ is everywhere differentiable and for any $u,v\in \XX$
$f( u ) \leq f(v) + \lr{\nabla f(v)}{ u - v } + \frac{L}{2} \| u - v \|^2$. 
An equivalent definition of smoothness is that $f$ has Lipschitz continuous gradient, i.e.,
$\| \nabla f(u) - \nabla f(v) \|_*\leq L \| u - v\|$.
\end{definition}

\begin{definition}
$f$ is $\sigma$-strongly convex w.r.t. a norm $\| \cdot \|$ if for any $u,v\in\XX$, $f( u ) \geq f(v) + \lr{\nabla f(v)}{ u - v } + \frac{\sigma}{2} \| u - v \|^2$ for some constant $\sigma>0$.
\end{definition}

\begin{definition}
For a convex function $f$, its Fenchel conjugate is $f^*(x) :=\displaystyle  \sup_{ y \in \text{dom} f } \,  \langle x, y \rangle - f(y)$.
\end{definition}
Note that if $f$ is convex then so is its conjugate $f^*$, since it is defined as the maximum over linear functions of $x$ (\cite{B04}). Morever, when the function $f(\cdot)$ is strictly convex and the above supremum is attained, we have that $\nabla f^*(x) = \displaystyle  \mathop{\arg\max}_y  \,  \langle x, y \rangle - f(y) $. 
Furthermore, the biconjugate $f^{**}$ equals $f$ if and only if $f$ is closed and convex.
It is known that $f$ is $\sigma$-strongly convex w.r.t. $\| \cdot \|$ if and only if $f^\ast$ is $1/\sigma$ strongly smooth 
with respect to the dual norm $\| \cdot \|_*$ (\cite{KST09}), assuming that $f$ is a closed and convex function.

\begin{definition}
A convex set $\K \subseteq \reals^d$ is a \emph{$\lambda$-strongly convex set} w.r.t. a norm $\| \cdot \|$ 
if for any $u, v \in \K$, any $\theta \in [0,1]$,
the $\| \cdot \|$ ball centered at $ \theta u + ( 1 - \theta) v$ with radius 
$\theta (1 - \theta) \frac{\lambda}{2} \| u - v \|^2$ is included in $\K$.
For examples of strongly-convex sets, we refer the readers to \citep{D15}. 	
\end{definition}

\begin{definition}
Let $\K$ be any closed convex set which contains the origin.
Then the \emph{gauge function} of $\K$ is $\g(x) := \inf \{ c \geq 0: \frac{x}{c} \in \K \}$
\end{definition}
One can show that the gauge function is a convex function (e.g. \cite{R96}). 
It is known that several closed convex sets can lead to the same gauge function (\cite{B13}).
But if a closed convex set $\K$ contains the origin, then the gauge function is unique and one has
$\K = \{ x \in \reals^d: \g(x)\leq 1 \}$.
Furthermore, $\text{int } \K = \{ x \in \reals^d: \g(x)< 1 \}.$

Next we provide a  characterization of sets based on their gauge function.
\begin{definition}[$\beta$-Gauge set]
Let $\K$ be a closed convex set which contains the origin.
We say that $\K$  is $\beta$-Gauge if its squared gauge function, $\g^2(\cdot)$,  is $\beta$-strongly-convex.
\end{definition}
This property captures a wide class of constraints. Among these are $l_p$ balls, 
Schatten $p$ balls, and the Group $(s,p)$ ball. We refer the reader to Appendix~{\ref{app:betagauge}} for more details. Curiously, all of these Gauge sets are also known to be strongly-convex. We conjecture that strong-convexity and the Gauge property are equivalent.


\section{Minimizing Regret to Solve Games}

Let us now turn our attention to a now-classical trick: using sequential no-regret learning strategies to find equilibria in zero-sum games.

\subsection{Weighted Regret Minimization}

We begin by briefly defining the standard online learning setup. We imagine a learner who must make a sequence of decisions, selecting at each round $t$ a point $x_t$ that lies within a convex and compact \emph{decision space} $\K$.
After selecting $x_t$ she is charged $\ell_t(x_t)$ for her action, where $\ell_{t}(\cdot)$ is the \emph{loss function} in round $t$, and she proceeds to the next round.
Typically it is assumed that when the learner selects $x_t$ in round $t$, she has observed all loss functions $\ell_1(\cdot), \ldots, \ell_{t-1}(\cdot)$ up to, but not including, time $t$. However, we will also consider learners that are \emph{prescient}, i.e. that can choose $x_t$ with knowledge of the loss functions up to \emph{and including} time $t$.

The standard objective for adversarial online learning is the \emph{regret}, defined as the difference between the learner's loss over the sequence, discounted by the loss of the best fixed action in hindsight. However, for the purposes of this paper we consider a generalized notion which we call the \emph{weighted regret}, where every time period has an importance weight that can differ from round to round. More precisely, we assume that the learning process is characterized by a sequence of weights $\balpha := \alpha_1, \alpha_2, \ldots, \alpha_T$, where $\alpha_t > 0$ for every $t$. Now we define the \emph{weighted regret} according to
\[
\textstyle  \regret{} := \sum_{t=1}^T  \alpha_t  \ell_t(x_t) - \min_{x \in \XX} \sum_{t=1}^T  \alpha_t  \ell_t(x).
\]
(Note that when we drop the $\balpha-$, this implies that $\alpha_t = 1$ for all $t$). The sequence of $\alpha_t$'s can arbitrary, and indeed we will consider scenarios under which these weights can be selected in an online fashion, according to the observed loss sequence. The learners also observe $\alpha_t$ at the end of each round. Throughout the paper we will use $A_t$ to denote the cumulative sum $\sum_{s=1}^t \alpha_s$, and of particular importance will be the weighted average regret $\avgregret{} := \frac{\regret{}}{A_T}$.

\subsection{Algorithms}
In this section we present several of the classical, and a few more recent, algorithms with well-established regret guarantees. For the most part, we present these algorithms in \emph{unweighted} form, without reference to the weight sequence $\balpha$. In later sections we specify more precisely their weighted counterparts.

One of the most well-known online learning strategies is known as \FTRL (FTRL), in which the decision point $x_t$ is chosen as the ``best'' point over the previous loss functions, with some additional regularization penalty according to some convex $R(\cdot)$. Precisely, given a parameter $\eta > 0$, the learner chooses on round $t$ the point
\begin{equation}\label{eq:FTRL}
\textstyle    x_t = \argmin_{x \in \K} \{ \eta \textstyle \sum_{s=1}^{t-1} \ell_s(x) + R(x) \}.
\end{equation}
For convenience, let $\nabla_t$ be the gradient $\nabla \ell_t(x_t)$. If we assume that $R(\cdot)$ is a strongly convex function with respect to some norm $\| \cdot \|$, then a well-known 
regret analysis grants the following bound:
\begin{equation}\label{eq:FTRLbound}
  \textstyle \uregret{} \leq \frac D \eta +  \frac \eta 2 \sum_{t=1}^T \| \nabla_t \|_*^2,
\end{equation}
where $D := \sup_{x \in \K} R(x)$. With an appropriately-tuned $\eta$, one achieves $\uregret{} \leq \sqrt{D\sum_{t=1}^T \| \nabla_t\|_*^2}$, which is $O(\sqrt{T})$ as long the gradients have bounded norm. See, e.g., \cite{shalev2012online,H14,RS16} for further details on this analysis.

The \FTL (FTL) strategy minimizes the objective \eqref{eq:FTRL}, but without the regularization penalty; i.e. $x_t = \argmin_{x \in \K} \sum_{s=1}^{t-1} \ell_s(x)$. Another way to formalize this is to consider $\eta \to \infty$. Given that the above bound has a $\frac 1 \eta$ term, it is clear we can not simply apply the same analysis of \FTRL, and indeed one can find examples where linear regret is unavoidable \cite{cesa2006prediction,shalev2012online}. On the other hand, it has been shown that a strong regret guarantee is achievable even without regularization, as long as the sequence of loss functions are \emph{strongly convex}. In particular, \cite{KS09} show the following result:
\begin{lemma}[Corollary 1 from \cite{KS09}]\label{lemma:ogd_strCvx}
Let $\ell_1, . . . ,\ell_T$ be a sequence of functions such that for all $t \in [T]$, $\ell_t$ is $\sigma_t$-strongly convex. Assume that the FTL algorithm runs on this sequence and for each $t \in [T]$, let $v_t$ be in $\partial \ell_t(x_t)$. Then,
\begin{align}
\textstyle \sum_{t=1}^T \ell_t(x_t) - \min_x \sum_{t=1}^T \ell_t(x) 
\le \frac{1}{2}\sum_{t=1}^T \frac{\norm{v_t}^2}{\sum_{\tau=1}^{t}\sigma_\tau}
\end{align}
Furthermore, let $G = \max_t \norm{v_t}$ and assume that for all $t \in [T], \sigma_t \ge \sigma$. Then, the regret is bounded by $\frac{G^2}{2\sigma}(\log(T) + 1)$.
\end{lemma}
In the context of solving zero-sum games, the online learning framework allows for one of the two players to be prescient, so she has access to one additional loss function $\ell_t(\cdot)$ before selecting her $x_t$. In such a case it is much easier to achieve low regret, and we present three standard prescient algorithms:
\begin{align}
  \BR\quad\quad & \textstyle x_t = \argmin_{x \in \K} \ell_t(x),\\
  \BTL\quad\quad & \textstyle x_t = \argmin_{x \in \K} \sum_{s=1}^{t} \ell_s(x), \\
  \BTRL\quad\quad & \textstyle x_t = \argmin_{x \in \K} \sum_{s=1}^{t} \ell_s(x) + \frac 1 \eta R(x) .
\end{align}
Indeed it is easy to show that, for the first two of these prescient strategies, one easily obtains $\uregret{} \leq 0$ \citep{kalai2005efficient}. The regret of \BTRL is no more than $\frac 1 \eta \sup_{x,x' \in \K} R(x) - R(x')$. We also consider optimistic algorithms, which we discuss in \Cref{app:optimistic}.

\paragraph{Gauge Function FTRL.}\label{sec:gauge} While the analysis of \FTRL is natural and leads to a simple intuitive bound \eqref{eq:FTRLbound}, it requires  solving a non-linear optimization problem on each round even when the loss functions $\ell_t(\cdot)$ are themselves linear -- a very common scenario. 
From a computational perspective, it is often impractical to solve the  \FTRL objective. Nevertheless, in many scenarios a (computationally feasible) \emph{linear optimization oracle} is at hand.
In such instances, much attention has been focused on a \emph{perturbed} version of \FTL, where one solves the unregularized optimization problem but with a linear noise term added to the objective; there is much work analyzing these algorithms and we refer the reader to \citet{kalai2005efficient,cesa2006prediction,abernethy2014online} among many others. The main downside of such randomized approaches is that they have good expected regret but suffer in variance, which makes them less suitable in various reductions.

In this work, we introduce a family of \FTRL algorithms that rely solely on a linear oracle, and we believe this is a novel approach to online linear optimization problems. The restriction we require is that the regularizer $R(\cdot)$ is chosen as the \emph{squared gauge function} $\g^2(\cdot)$ for the decision set $\K$ of the learner. Here we will assume\footnote{One can reduce any arbitrary convex loss to the linear loss case
	by convexity $\ell_t(x) - \ell(x^*) \leq \langle \partial f_t(x), x - x^* \rangle$. 
	(\cite{,shalev2012online,H14,RS16}).} for every $t$ that $\ell_t(\cdot) = \langle l_t, \cdot \rangle$ for some vector $l_t$, hence the objective \eqref{eq:FTRL} reduces to
\begin{equation}\label{eq:gauge_FTRL}
    x_t = \argmin_{x \in \K} \eta \langle L_{t-1} , x \rangle + \g^2(x),
\end{equation}
where $L_{t-1} = l_1 + \ldots + l_{t-1}$. 
Denote $\text{bndry}(\K)$ as the boundary of the constraint set $\K$.
We can reparameterize the above optimization, by observing that any point $x \in \K$ can be written as $\rho z$ where $z \in \text{bndry}(\K)$, and $\rho \in [0,1]$. Hence we have 
\begin{equation}\label{eq:gauge_FTRL}
     \min_{\rho \in [0,1]} \min_{z \in \text{bndry}(\K)}  \eta \langle L_{t-1} , \rho z \rangle + \g^2(\rho z)
     ~=~
     \min_{\rho \in [0,1]} \left( \min_{z \in \text{bndry}(\K)} \eta \langle L_{t-1} , z \rangle \right) \rho  + \rho^2.
\end{equation}
We are able to remove the dependence on the gauge function since it is homogeneous, $\g(\rho x) = |\rho| \g(x)$, and is identically 1 on the boundary of $\K$. The inner minimization reduces to the linear optimization $z^* := \argmin_{z \in \K} \langle L_{t-1}, z \rangle$, and the optimal $\rho$ is $\max(0, \min(1, -(\eta/2) \langle L_{t-1}, z^* \rangle) )$.


\subsection{Solving zero-sum convex-concave games}

Let us now apply the tools described above to the problem of solving a particular class of zero-sum games; these are often referred to as convex-concave saddle point problems. Assume we have convex and compact sets $\XX \subset \reals^n$, $\YY \subset \reals^m$, known as the \emph{action spaces} for the two players. We are given a convex-concave \emph{payoff function} $g: \XX \times \YY$; that is, $g(\cdot, y)$ is convex in its first argument for every fixed $y \in \YY$, and $g(x,\cdot)$ is concave in its second argument for every fixed $x \in \XX$. We say that a pair $(\hat x, \hat y) \in \XX\times \YY$ is an \emph{$\epsilon$-equilibrium} for $g(\cdot, \cdot)$ if $\sup_{y \in \YY} g(\hat x, y) - \inf_{x \in \XX} g(x, \hat y) \leq \epsilon.$.

The celebrated \emph{minimax theorem}, first proven by von Neumann for a simple class of biaffine payoff functions \citep{v1928theorie,neumann1944theory} and generalized by \citet{sion1958general} and others, states that there exist 0-equilibria for convex-concave games under reasonably weak conditions. Another way to state this $\inf_{x \in \XX} \sup_{y \in \YY} g(x, y) =  \sup_{y \in \YY}\inf_{x \in \XX} g(x, y)$, and we tend to call this quantity $V^*$, the \emph{value of the game} $g(\cdot, \cdot)$.

The method of computing an $\epsilon$-equilibrium using a pair of no-regret algorithms is reasonably straightforward, although here we will emphasize the use of weighted regret, which has been much less common in the literature. Algorithm~\ref{alg:game} describes a basic template used throughout the paper. 
\begin{algorithm}[h] 
   \caption{ Computing equilibria using no-regret algorithms } \label{alg:game}
\begin{algorithmic}[1]
\STATE Input: a $T$-length sequence $\balpha$
\FOR{$t= 1, 2, \dots, T$}
\STATE $x$-player selects $x_t \in \XX$ using no-regret algorithm $\alg^x$
\STATE $y$-player selects $y_t\in \YY$ using (possibly-different) no-regret algorithm $\alg^y$
\STATE $x$-player suffers loss $\ell_{t}(x_t)$ with weight $\alpha_t$, where $\ell_t(\cdot) = g(\cdot,y_t)$
\STATE $y$-player suffers loss $h_{t}(y_t)$ with weight $\alpha_t$, where $h_t(\cdot) = -g(x_t,\cdot)$
\ENDFOR
\STATE Output $(\bar{x}_{\balpha},\bar{y}_{\balpha}) := \left(\frac{ \sum_{s=1}^T \alpha_s x_s  }{ A_T }, \frac{ \sum_{s=1}^T \alpha_s y_s  }{ A_T }\right)$
\end{algorithmic}
\end{algorithm}





\begin{theorem} \label{thm:convergence}
	Assume that a convex-concave game payoff $g(\cdot, \cdot)$ and a $T$-length sequence $\balpha$ are given. Assume that we run Algorithm~\ref{alg:game} using no-regret procedures $\alg^x$ and $\alg^y$, and the $\balpha$-weighted average regret of each is $\avgregret{x}$ and $\avgregret{y}$, respectively. Then the output  $(\bar{x}_{\balpha},\bar{y}_{\balpha})$ is an $\epsilon$-equilibrium for $g(\cdot, \cdot)$, with $\epsilon = \avgregret{x} + \avgregret{y}.$
\end{theorem}
The theorem can be restated in terms of $V^*$, where we get the following ``$\epsilon$-sandwich'':
\begin{equation} \label{eq:xyguarantee}
\textstyle   V^* - \epsilon
   \leq \inf_{x\in X} g \pr{x, \bar{y}_{\balpha}} 
   \leq V^* 
   \leq \sup_{y\in \YY} g\pr{ \bar{x}_{\balpha}, y}  
   \leq  V^* + \epsilon
\end{equation}
But the key insight is that the regret analysis leads immediately to a convergence rate for the algorithmic template presented above. We provide the proof of \Cref{thm:convergence} in Appendix~\ref{app:thm:convergence}.
\subsection{Application: the Frank-Wolfe Algorithm}
We can tie the above set of tools together with an illustrative application, describing a natural connection to the Frank-Wolfe (FW) method \citep{frank1956algorithm} for constrained optimization. The ideas presented here summarize the work of \citet{AW17},  but in \Cref{sec:fastfw} we significantly strengthen the result for a special case.

We have a convex set $\K$, an $L$-smooth convex function $f : \K \to \reals$, and some initial point $w_0 \in \K$.
The FW algorithm makes repeated calls to a linear optimization oracle over $\K$, followed by a convex averaging step:
\begin{equation*}
\begin{aligned}
\textstyle(\textbf{linear opt}) \quad v_t = \argmin_{v \in \K} \langle v, \nabla f(w_{t-1})  \rangle; \quad \quad (\textbf{update}) \quad  w_{t} = (1 - \eta_t) w_{t-1} + \eta_t v_t,
\end{aligned}
\end{equation*}
where the parameter $\eta_t$ is a learning rate, and following the standard analysis one sets $\eta_t = \frac 2 {t + 2}$. A well-known result is that $f(w_T) - \inf_{w \in \K} f(w) \leq {2L D^2}/{T}$.

Let us leverage Theorem~\ref{thm:convergence} to obtain a convergence rate from a no-regret perspective. With a brief inspection, one can verify that FW is indeed a special case of Algorithm~\ref{alg:game}, assuming that (a) the game payoff is $g(x,y) := f^*(x) - \langle x, y \rangle$, where $f^*$ is the Fenchel conjugate of $f$; (b) the sequence $\balpha$ is $1, 2, \ldots, T$; (c) the $x$-player and $y$-player employ \FTL and \BR, respectively; we output the final iterate as $w_T := \bar y_{\balpha}$. We refer to \citet{AW17} for a thorough exposition, but it is striking that this use of Algorithm~\ref{alg:game} leads to Frank-Wolfe even up to $\eta_t = \frac{2}{t+2}$.

As we have reframed FW in terms of our repeated game, we can now appeal to our main theorem to obtain a rate. We must first observe, using the duality of Fenchel conjugation, that 
\begin{align}
\textstyle	V^* = \sup_y \inf_x  g(x,y) = - (\inf_y \sup_x  \langle x, y \rangle - f^*(x)) = - (\inf_y f(y)).\label{eq:fwgamevalue}
\end{align}
Using \eqref{eq:xyguarantee} and the above equality, we can obtain $f(\bar y_{\balpha}) \leq \inf_{y \in \K} f(x) + \avgregret{x} + \avgregret{y}$.

The convergence rate of FW thus boils down to bounding the regret of the two players. We note first that the $y$-player is prescient and employs \BR,  hence we conclude that $\avgregret{y} \leq 0$. The $x$-player on the other hand will suffer the $\balpha$-weighted regret of \FTL. But notice, critically, that the choice of payoff $g(x,y) = f^*(x) - \langle x, y \rangle$ happens to be \emph{strongly convex} in $x$, as $L$-smoothness of $f$ implies $L^{-1}$-strong convexity in $f^*$. We may thus use Lemma~\ref{lemma:ogd_strCvx} to obtain:
\[
	\textstyle \avgregret{x} \overset{Lemma~\ref{lemma:ogd_strCvx}}{\leq} \frac 1 {2 A_T} \sum_{t=1}^T \frac{ \alpha_t^2 D }{ \sum_{\tau=1}^t \alpha_{\tau} (1/L)}  \leq  O( \sum_{\tau=1}^T L D^2  /A_T)  = O( \frac{LD^2}{T} ),
 \]
where we use the fact that the x-player observes an $\alpha_{t}/L$ strongly convex function, $\alpha_{t} \ell_{t}(\cdot)$, and that
$\| v_{t}\|^{2}$ in Lemma~\ref{lemma:ogd_strCvx} is $\| \alpha_{t}  \nabla \ell_t(\cdot)\|^{2}\leq \alpha_{t}^{2} D^{2}$, where $D$ is the diameter of $\YY$.
We conclude by noting that the absence of the $\log T$ term, which tends to arise from the regret of online strongly convex optimization, was removed by carefully selecting the sequence of weights $\balpha$.

\section{Fast convergence in the FW game}\label{sec:fastfw}
In this section, we introduce a new FW-like algorithm that achieves a $O(1/T^2)$ convergence rate on $\beta$-Gauge  sets $\YY$ accessed using a linear optimization oracle. The design and analysis are based on a reweighting scheme and \OFTRL, taking advantage of recent tools developed for fast rates in solving games \citep{CJ12,RS13,SALS15}.

In Theorem~\ref{thm:FWgame} we give an instantiation of \Cref{alg:game} that finds an approximate saddle point for the FW game $g(x,y) = f^*(x) - \lr{x}{y}$. In this  instantiation  the $x$-player plays 
\OFTL and the $y$-player plays \BTRL. With an appropriate weighting, the weighted regret guarantees of these two algorithms imply that we can find an $O(\frac{1}{T^2})$-approximate saddle point solution of the FW game in $T$ rounds.
Recalling that $\min_{x\in\XX}\{ f^*(x) - \lr{x}{y}\} = -f(y)$,  this 
  immediate translates to a convergence rate of  $O(\frac{1}{T^2})$ for the the problem $\min_{y\in\YY}f(y)$.

The algorithm that we describe in 
Theorem~\ref{thm:FWgame} does not immediately yield a FW-like algorithm---in general, we may not be able to compute the $y$-player's \BTRL iterates using only a linear optimization oracle. However, if the $y$-player uses the squared gauge function of $\YY$ as a regularizer, then the $y$ iterates are computable using a linear optimization oracle, as shown in \Cref{sec:gauge}. This fact immediately implies that for $\beta$-Gauge sets
and upon choosing the gauge function as regularizer, \Cref{alg:FW} instantiates a \emph{projection-free} procedure which provides a convergence rate of $O(1/T^2)$ for the problem  $\min_{y\in\YY}f(y)$ (see Corollary~\ref{cor:FWcor}).
In Appendix~\ref{app:BTPL}, we discuss how to get a faster rate than $O(1/T)$ for arbitrary convex sets if \BTPL rather than \BTRL is used by the $y$-player in the FW game.

\subsection{Solving the FW game with \OFTL and \BTRL}
In this section, we present our algorithm for finding $O(1/T^2)$-saddle point solutions to the FW game. We instantiate \Cref{alg:game} using the FW objective $g(x,y) = f^*(x) - \lr{x}{y}$, where we assume $f$ is $L$-smooth and $\sigma$-strongly convex. The $x$-player plays \OFTL and the $y$-player plays \BTRL.
\begin{theorem}~\label{thm:FWgame}
Assume that we instantiate \Cref{alg:game} with the FW game $g(x,y) = f^*(x) - \lr{x}{y}$, weight sequence $\alpha_t=t$, and the following strategies for the players. The $x$-player plays \OFTL: 
\begin{align}\label{eq:XplayerFWgame}
  \textstyle x_t = \argmin_{x \in \XX} \sum_{s=1}^{t-1} \alpha_s \ell_s(x) + m_{t}(x) \text{ with } m_{t}(x) = \alpha_{t} \ell_{{t-1}}(x)
\end{align}
where $\ell_t(x) = g(x,y_t)$, and the $y$-player plays \BTRL: 
\begin{align}\label{eq:YplayerFWgame}
\textstyle y_t = \argmin_{y \in \YY} \sum_{s=1}^{t} \alpha_s h_s(y) + \frac{1}{\eta} R(y)
\end{align}
with a $\beta$-strongly-convex regularizer $R(\cdot)$ and $\eta = {\beta }/{ 16L ( 1 + \frac{L}{\sigma})}$, where $h_t(x) = -g(x_t,y)$. Then the output $(\bar{x}_{\balpha},\bar{y}_{\balpha})$ of   \Cref{alg:game} is 
an $O\pr{ \frac{ L ( R(y^*) - R(z) )(1+\frac{L}{\sigma}) }{ \beta T^2}}$
-approximate saddle point solution to the FW game,
where $z = \arg\min_{{y\in \YY}} R(y) $.
\end{theorem}

Now recall that for the FW setting, we are interested in $y$-players that may only employ a linear optimization oracle. In general it is impossible to solve Equation~\eqref{eq:YplayerFWgame} within $O(1)$ calls to such oracles in each round. Nevertheless, recall that for $\beta$-Gauge sets, choosing $R(y) = \gamma_{\YY}^2(y)$ induces a 
$\beta$-strongly-convex regularizer, while enabling us to solve Equation~\eqref{eq:YplayerFWgame} with a single call to the linear oracle, as shown in Equation~\ref{eq:gauge_FTRL}. The proof of \Cref{thm:FWgame} shows that the $x$-player's strategy is the gradient of the primal objective $f$ at the point $\bar{y}_{\balpha'_{1:t-1}}$, where $\balpha'$ is a weight vector such that $\alpha'_s$ = $\alpha_s$ for $s = 1,...,t-1$ and $\alpha'_{t-1} = \alpha_{t-1} + \alpha_t$ and $\bar{y}_{\balpha'_{1:t-1}}$ is the $\balpha'$-weighted average of $y_1,...,y_{t-1}$ (See Equation~\ref{xt}). This leads to \Cref{alg:FW} and Corollary~\ref{cor:FWcor}.

\begin{algorithm}[h] 
	\caption{A new FW algorithm}
	\label{alg:FW}
	\begin{algorithmic}[1]
		\STATE Let $\balpha$ be a $T$-length weight sequence 
		\FOR{$t= 1, 2, \dots, T$}
		\STATE Define the $T$-length sequence $\balpha'$ as $\alpha_{s}' = \alpha_s $ for $s=1,\ldots, t-1$, and $\alpha_{t}' = \alpha_{t-1} + \alpha_t$
		\STATE  Set $x_t = \nabla f( \frac{\Sigma_{{s=1}}^{t-1} \alpha_s' y_s}{ A_t})$
		\STATE  Set	$(\hat{y}_t, \rho_t) = \underset{y \in \YY, \rho \in[0,1] }{\arg\min} \sum_{s=1}^t \rho \langle  y, \alpha_s x_s \rangle  +  \frac{1}{\eta} \rho^2$\quad for $\eta = \frac{\beta }{ 16L ( 1 + \frac{L}{\sigma})}$
		\STATE Set $y_t = \rho_t \hat{y}_t$                  
		\ENDFOR
		\STATE Output $\bar{y}_{T} := \frac{ \sum_{s=1}^T \alpha_s y_s  }{ A_T },$ where $A_{T}= \sum_{s=1}^T \alpha_{t}$
	\end{algorithmic}
\end{algorithm}

We get the following corollary of the above theorem. The full proof is in \Cref{app:fwcor}. 

\begin{corollary}\label{cor:FWcor} \label{thm:FW}
Let $f:\YY \mapsto \reals$ be $L$-smooth and $\sigma$-strongly-convex. Also assume that $\YY$ is a $\beta$-Gauge set. Let $\alpha_t = t$. Then the output $\bar{y}_{T}$ in Algorithm~\ref{alg:FW} is an $O\pr{ \frac{ L ( R(y^*) - R(z) )(1+\frac{L}{\sigma}) }{ \beta  T^2}}$-approximate optimal solution to the optimization problem $\min_{y \in \YY} f(y)$. Moreover, \Cref{alg:FW} only requires a single linear optimization oracle call in each round.
\end{corollary}

\subsection{Proof of \Cref{thm:FWgame}}
\begin{proof}[Proof of Theorem~\ref{thm:FWgame}]
	In the FW game,
	we observe that the loss functions $\alpha_{t} \ell_{t}(\cdot)$ seen by the x-player are $\frac{\alpha_t}{L}$-strongly convex,
	since the function $f(\cdot)$ is $L$ smooth, which implies that $f^{*}(\cdot)$ is $\frac{1}{L}$-strongly convex.
	
	The $x$-player chooses $x_{t}$ based on \OFTL:
	$x_t = \argmin_{x \in \XX} \sum_{s=1}^{t-1} \alpha_s \ell_s(x) + m_{t}(x)$,
	where $m_{t}(x) = \alpha_{t} \ell_{{t-1}}(x)$.
	To analyze the regret of the $x$-player, let us first denote 
	the update of the standard \FTL as
	\begin{equation}
\textstyle	z_t = \argmin_{x} \sum_{s=1}^{t-1} \alpha_s \ell_s(x).
	\end{equation}
	
	Denote $x^{*} := \arg\min_{x \in \XX} \sum_{t=1}^T  \alpha_t  \ell_t(x)$.\footnote{The following analysis actually holds for any $x^{*} \in \XX$.}
	Now we are going to analyze the $\balpha$-weighted regret of the $x$-player, which is 
	\begin{equation}
	\begin{aligned}
	\regret{x} & \textstyle := \sum_{t=1}^{T} \alpha_t \ell_t(x_t) - \alpha_t \ell_t(x^*)
	\\ & \textstyle = \sum_{t=1}^{T} \alpha_t \ell_t(x_t)  - \alpha_t \ell_t(z_{t+1}) - m_t(x_t) + m_t(z_{t+1})
	+ \sum_{t=1}^{T} m_t(x_t) - m_t(z_{t+1})
	\\ & \textstyle \quad \quad + \sum_{t=1}^T \alpha_t \ell_t(z_{t+1}) - \alpha_t \ell_t(x^*).
	\\ & \textstyle\leq \sum_{t=1}^{T} \lr{\alpha_t \nabla \ell_t(x_t) - \nabla m_t(z_{t+1}) }{x_t - z_{t+1}}
	+ \sum_{t=1}^{T} m_t(x_t) - m_t(z_{t+1})
	\\& \textstyle \quad \quad  + \sum_{t=1}^T \alpha_t \ell_t(z_{t+1}) - \alpha_t \ell_t(x^*) - \frac{2 \alpha_t}{L} \| x_t - z_{t+1}\|^2,  \label{3sums}
	\end{aligned}
	\end{equation}
	where the last inequality uses strong convexity of $\ell_{t}(\cdot)$ so that 
	\begin{equation}
\textstyle	\alpha_t \ell_t(x_t)  - \alpha_t \ell_t(z_{t+1}) \leq \lr{\alpha_t \nabla \ell_t(x_t)}{x_t- z_{t+1}} - \frac{\alpha_t}{L} \| x_t - z_{t+1}\|^2,
	\end{equation}
	and that 
	\begin{equation}
\textstyle	- m_t(x_t) + m_t(z_{t+1})
	\leq \lr{\nabla m_t(z_{t+1})}{z_{t+1}-x_t} - \frac{\alpha_t}{L} \| x_t - z_{t+1}\|^2.
	\end{equation}
	
	There are three sums in (\ref{3sums}). 
	Note that the second sum should be small because the expression for $x_t$ ``exploits'' $m_{t}(\cdot) = \alpha_t \ell_{t-1}(\cdot)$ more than the expression for $z_{t-1}$ does. The third sum is the regret of \BTL, which is non-positive. 
	In Lemma~\ref{2in3sums}, we show that the second and third sums in \cref{3sums} are in total non-positive. For the proof, please see Appendix~\ref{app:2in3sums}. 
	
	Since $m_{t}(\cdot) := \alpha_{t} \ell_{t-1}(\cdot)$,
	each term in the first sum in (\ref{3sums}) can be bounded by
	\begin{equation*} 
	\begin{aligned}
\textstyle  \lr{\alpha_t \nabla \ell_t(x_t) - \nabla m_t(z_{t+1})}{x_t - z_{t+1}}
	 & \textstyle = \alpha_t \lr{ \nabla \ell_t(x_t) - \nabla \ell_{t-1}(z_{t+1})}{x_t - z_{t+1}}
	\\ & \textstyle = \alpha_t \lr{ - y_t + \nabla f^*(x_t) + y_{t-1} - \nabla f^*(z_{t+1}) }{x_t - z_{t+1}}
	\end{aligned}
	\end{equation*}\vspace{-8mm}
	\begin{align}
\textstyle	 \quad\quad&\leq \alpha_t  \pr{ \| y_t - y_{t-1} \|_* \| x_t- z_{t+1} \| + \frac{1}{\sigma} \| x_t - z_{t+1} \|^2 },\label{tmp1}
	\end{align}
	where the last inequality uses  
	H\"older's inequality and 
	the fact that $f$ is $\sigma$-strongly convex
	so that $f^{*}$ is $\frac{1}{\sigma}$ smooth.
	Let us analyze $\| x_t- z_{t+1} \|^{2}$. Note that, by Fenchel conjugacy,
$\textstyle	z_{t+1} := \arg\min_x \sum_{s=1}^t \alpha_s ( -y_s^\top x + f^*(x)) = \nabla f(\bar{y}_{\balpha_{1:t}})$,
	where $\bar{y}_{\balpha_{1:t}}$ is the $\balpha$-weighted average of $y_1, \ldots, y_t$
	For notational simplicity, let us define a new weight vector $\balpha'$, where $\alpha'_s$ = $\alpha_s$ for $s = 1,...,t-1$ and $\alpha'_{t-1} = \alpha_{t-1} + \alpha_t$. 
		Similarly, for $x_{t}$, we have 
	\begin{equation} \label{xt}
\textstyle	x_t := \arg\min \{ \alpha_t ( -y_{t-1}^\top x + f^*(x)) +  \sum_{s=1}^{t-1} \alpha_s ( -y_s^\top x + f^*(x)) \}  = \nabla f( \bar{y}_{\balpha'_{1:t-1}}),
	\end{equation}
	where $\bar{y}_{\balpha'_{1:t-1}}$ is the $\balpha'$-weighted average of $y_1, \ldots, y_{t-1}$.
	According to \eqref{xt}, 
	\begin{equation*}
	\begin{aligned}
\textstyle	\| x_t- z_{t+1} \|^{2}  & \textstyle = \| \nabla f(\bar{y}_{\balpha_{1:t}}) -  \nabla f(\bar{y}_{\balpha'_{1:t-1}}) \|^2 
	\leq L^2 \| \bar{y}_{\balpha_{1:t}} -  \bar{y}_{\balpha'_{1:t-1}} \|^2
	\\\textstyle &= \frac{L^2}{A_t^2} \|  \sum_{s=1}^t \alpha_s y_s - \sum_{{s=1}}^{t-1} \alpha_s' y_s  \|^2 = \frac{L^2}{A_t^2} \| \alpha_{t-1} y_{t-1} + \alpha_t y_t - \alpha'_{t-1} y_{t-1} \|^2
	\end{aligned}
	\end{equation*}\vspace{-8mm}
	\begin{align}
	&= \textstyle \frac{L^2}{A_t^2} \left\|  \alpha_t ( y_{t-1} - y_t ) \right\|^2
	= \pr{\frac{\alpha_t L }{A_t}}^2 \|y_{t-1} - y_t \|^2. \label{tmp2}
	\end{align}	
	Combining $(\ref{tmp1})$ and $(\ref{tmp2})$, we get
	\begin{equation}
	\begin{aligned}
	& ( \alpha_t \nabla \ell_t(x_t) - \nabla m_t(z_{t+1})  )^\top (x_t - z_{t+1})
	\leq \alpha_t  ( \| y_t - y_{t-1} \|_* \| x_t- z_{t+1} \| + \frac{1}{\sigma} \| x_t - z_{t+1} \|^2 ) 
	\\ & \textstyle\leq \alpha_t  \pr{\pr{\frac{\alpha_t L }{A_t}}  \| y_t - y_{t-1} \|^2 + \frac{1}{\sigma} \pr{\frac{\alpha_t L}{A_t}}^2  \| y_t - y_{t-1} \|^2 }.
	\end{aligned}
	\end{equation}
	Therefore, we have shown that the first sum in (\ref{3sums}) is bounded by

	\begin{equation} \label{1sum}
	\begin{aligned}
	\textstyle
	\sum_{t=1}^T
	\alpha_t  \pr{ \pr{\frac{\alpha_t L}{A_t}}  \| y_t - y_{t-1} \|^2 	+  \frac{1}{\sigma} \pr{\frac{\alpha_t L }{A_t}}^2  \| y_t - y_{t-1} \|^2 }.
	\end{aligned}
	\end{equation}
	By (\ref{3sums}), (\ref{1sum}), and Lemma \ref{2in3sums}, we get the upper bound of the regret of the $x$-player,
	\begin{equation} \label{regret_x}
\textstyle	\regret{x} \leq \sum_{t=1}^T
	\alpha_t  \pr{\pr{\frac{\alpha_t L}{A_t}}  \| y_t - y_{t-1} \|^2 + \frac{1}{\sigma} \pr{\frac{\alpha_t L }{A_t}}^2  \| y_t - y_{t-1} \|^2 }.
	\end{equation}
	
	Now let us switch to analyze the regret of the $y$-player, which is defined as
$\textstyle	\regret{y}   :=  \sum_{t=1}^{T} - \alpha_t h(y_t) + \alpha_t h_t(y^*)
	= \sum_{t=1}^{T} -\alpha_t  ( - \langle x_t, y_t \rangle + f^*(x_t) ) + \alpha_t (- \langle x_t, y^* \rangle + f^*(x_t) ),$
	which equals $\sum_{t=1}^{T} \langle y_t - y^{*}, \alpha_t x_{t} \rangle$.
	This means that the $y$-player actually observes the linear loss $\alpha_t x_{t}$ in each round $t$, due to the fact that the $y$-player plays after the $x$-player plays. 
	We can reinterpret \BTRL as \OFTRL (\cite{SALS15}) when the learner is fully informed as to the loss function for the current round. That is, we may write the update as
	$y_{t} = \arg\min_{y \in \YY} \langle y , \sum_{s=1}^{t-1} \alpha_s x_{s} + m_t \rangle + \frac{1}{\eta} R(y)$, where $m_{t}:=\alpha_t x_{t}$ and $R(\cdot)$ is $\beta$-strongly convex with respect to a norm $\|\cdot\|$ on $\YY$.
	
	For loss vectors $\theta_t$, \Cref{thm:oftrl} gives the regret of \OFTRL as 
	\begin{equation} 
\textstyle\sum_{t=1}^{T} \langle y_t - y^{*}, \theta_{t} \rangle 
	\leq \frac{ R(y^*) - R(z)- 
\frac{\beta}{2} (\sum_{{t=1}}^{T} \| y_t - z_t \|^2 + \sum_{{t=1}}^{T} \| y_{t} - z_{t+1} \|^2)}{\eta} +
	\sum_{t=1}^{T} \frac{\eta}{\beta} \| \theta_t - m_{t} \|^2_* ,
	\end{equation}
	where $z_{{t}}$ is FTRL update, defined as $z_{t} = \arg\min_{z \in \YY}  \langle z, \sum_{s=1}^{t-1} \theta_s \rangle +  \frac{1}{\eta} R(y)$, while $y_{t}$ is \OFTRL update, defined as 
	$y_{t} = \arg\min_{z \in \YY}  \langle z, ( \sum_{s=1}^{t-1} \theta_s )  + m_t \rangle  + \frac{1}{\eta} R(y)$.
     We prove \Cref{thm:oftrl} in Appendix~\ref{app:optFTRL}.
	
	Since in our case $\theta_t = m_{t} = \alpha_{t} x_{t}$, $y_{t}=z_{{t+1}}$
	we get the bound of the regret of the $y$-player in the FW game,
	\begin{equation} \label{regret_y}
\textstyle	\regret{y} \leq \frac{  R(y^*) - R(z) - \frac{\beta}{2} \sum_{{t=1}}^{T} \| y_{t} - y_{t-1} \|^2}{\eta}.
	\end{equation}
	Combining (\ref{regret_x}) and (\ref{regret_y}),
	we get 
	\begin{equation} \label{regretxy}
	\begin{aligned}
\textstyle	\regret{x} + \regret{y} 
\textstyle	\leq & \textstyle \sum_{t=1}^T \textstyle \bigg\{
	\alpha_t  \pr{\pr{\frac{\alpha_t L}{A_t}} \| y_t - y_{t-1} \|^2 + \frac{1}{\sigma} \pr{\frac{\alpha_t L }{A_t}}^2  \| y_t - y_{t-1} \|^2 } \\ & \textstyle + \frac{ R(y^*) - R(z) - \frac{\beta}{2} \sum_{{t=1}}^{T} \| y_{t+1} - y_t \|^2}{\eta} \bigg\}
	\end{aligned}
	\end{equation}
	The coefficient of $\| y_t - y_{t-1} \|^2$ is $\frac{\alpha_t^2 L}{A_t} +  \frac{\alpha_t}{\sigma} (\frac{\alpha_t L }{A_t})^2 - \frac{\beta}{2 \eta }$ and, as $\alpha_{t}=t$ and if we set $\eta =\frac{\beta }{ 16L ( 1 + \frac{L}{\sigma})}$, the quantity becomes negative.
	So,
	\begin{equation} \label{theRate}
\textstyle	\avgregret{x} + \avgregret{y}
	\leq  \frac{
		\regret{x} + \regret{y} }{  A_t}
	= O\pr{ \frac{ L ( R(y^*) -R(z) )(1+\frac{L}{\sigma}) }{ \beta T^2}}.
	\end{equation}
	Combining this with Theorem~\ref{thm:convergence} completes the proof. \end{proof}

\section{Linear convergence in some strongly-convex and strongly-concave game}

The $O(1/T^2)$ rate in the previous section is nice, as it shows that one can achieve ``acceleration'' under certain conditions using no-regret learning dynamics. But when can we get even faster rates using the template Algorithm~\ref{alg:game}? In the present section we describe scenarios in which the two players can compute an approximate equilibrium in \emph{linear time}, i.e. with a rate exponentially decaying in $T$. This requires stronger assumptions on the payoff $g(\cdot,\cdot)$, and we will explore three such scenarios. But at a minimum we assume, throughout this section, that $g(x,y)$ is (a) $\sigma_x$-strongly convex in $x$, (b) $L$-smooth in $x$ for every $y \in \YY$.


The key to obtaining a fast rate is to consider the function $s(x) = \sup_{y \in \YY}\; g(x,y).$ This function reports the payoff/loss value given to $x$ when the $y$-player plays the best response $\yx{x} = \arg\max_{y \in \YY} g(x,y)$; i.e. $s(x) = g(x,\yx{x})$. We make two important insights about $s$. First, the supremum of strongly convex functions remains strongly convex, so $s$ is $\sigma_x$ strongly convex in $x$. Second, we make an important insight, established below, that the gradient of $g(\cdot, y_x)$ at the point $x$ must be a subgradient of $s(\cdot)$ also at $x$. This is formally stated and proven in \Cref{app:smoothH-1}.

The key property needed for $s(\cdot)$ however, is \emph{smoothness}. It is easy to construct games for which $g(\cdot,\cdot)$ is smooth in a very strong sense, but the resulting $s(\cdot)$ is non-smooth. We need additional structure, and we give three scenarios where the condition holds. $s(\cdot)$ is smooth when any of the following are true:
\begin{enumerate}[itemsep=0mm]
    \item We can write $g( x, y ) = a(x) +  x^\top M y - b(y)$, where $a(x)$ is any $\sigma_x$-strongly convex function of $x$, $M$ is a matrix, and $b(y)$ is any $\sigma_y$-strongly convex function of $y$; (See \Cref{app:smoothH-1})
  \item The function $g(x,y)$ is $\sigma_y$-strongly-concave in $y$, the best response function $y_x := \argmax_{y \in \YY}$ $g(x,y)$ is always achieved in the interior of $\YY$, and we have that for every $w,z \in \XX$,  $\| \nabla_{\yx{z}} g(w,\cdot) - \nabla_{\yx{z}} g(z,\cdot)   \| \le L \| w - z \|$; (See \Cref{app:smooth2})
  \item Assume we have the Frank-Wolfe game, where $g(x,y) := f^*(x) - \langle x, y \rangle$, the objective $f(\cdot)$ is smooth, and we are guaranteed that $\|\nabla f(y)\|_2$ is lower bounded by some positive constant $B$ for every $y \in \YY$.
\end{enumerate}

Any of the above conditions is suitable to obtain our main result of this section, which is that the framework of no-regret equilibrium computation allows one to obtain a \emph{linear convergence rate} to solve certain games. This combines a careful analysis of the game, with a recent online-to-batch conversion trick given by \citet{L17}. The key tool is to set the sequence of weights according to the inverse squared norm of the loss gradients. The result relies on the $x$-player using an algorithm known as SC-AdaGrad, a simple adaptive gradient descent procedure (\Cref{alg:SC-AdaNGD}). We postpone the description of this algorithm, and the proof of the following theorem, to \Cref{app:linear-rate}.

\begin{theorem}\label{thm:linear-rate}
Assume $g$ satisfies any conditions such that the resulting $s(\cdot)$ is $L$-smooth and $\sigma$-strongly convex function for constants $L$ and $\sigma$ and moreover that $(\arg\min_{x\in \reals^d} s(x)) \in \XX$. Suppose we instantiate Algorithm~\ref{alg:game}, where $x$-player uses SC-AdaGrad (\Cref{alg:SC-AdaNGD}) for loss functions $\alpha_t \ell_t(\cdot) = \alpha_t g(\cdot, y_t)$, $y$-player uses \BR, and the sequence $\balpha$ is defined as $\alpha_t := \| \nabla \ell_t(x_t) \|^{-2}$. Let $T \geq \frac L \sigma \log 3$, and constant $G$ upper bounded $\|\nabla \ell_t\|$. Then the output $(\bar x_{\balpha}, \bar y_{\balpha})$ is an $\epsilon$-equilibrium of $g$ where $\epsilon = O\pr{\frac{G^2 T}{L}e^{-\frac{\sigma}{L}T}}$.
\end{theorem}

Notice that one of our scenarios stated above is the Frank-Wolfe game with an additional strong convexity assumption on the objective $f$, as well as a necessary lower bound on the norm of $\nabla f$ within the feasible region. That we achieve a linear rate for this scenario is indeed not surprising, as this was previously established \citep{LP66,DR70,D79}. Unfortunately, \Cref{thm:linear-rate} does \emph{not} provide an FW-like algorithm, as the combination of the SC-AdaGrad and \BR subroutines do not reduce to a simple linear optimization. An alternative algorithm we propose called SC-AFTL (\Cref{alg:SC-AFTL}), which is akin to an adaptive version of \FTL, does reduce to a linear oracle. We state the main theorem below, yet its proof and \Cref{alg:SC-AFTL} are described in full detail in \Cref{app:linearFW}.


\begin{theorem}\label{thm:linearFW} 
Consider the FW game in which $g(x,y)= -\lr{x}{y} + f^{*}(x)$.
Suppose that $f(\cdot)$ is a $L$-smooth convex function
and that $\YY$ is a $\lambda$-strongly convex set. Also assume that the gradients of the $f$ in $\YY$ are bounded away from $0$, i.e., $\max_{y\in\YY}\|\nabla f(y)\|\geq B$.
Then, there exists a FW-like algorithm that has $O(\exp(-\frac{\lambda B }{L} T))$ rate.
\end{theorem}

\acks{K.Y.L. is supported by the ETH Zurich Postdoctoral Fellowship and Marie Curie Actions for People COFUND program. We gratefully acknowledge financial support from the National Science Foundation, award IIS 1453304.}

\bibliographystyle{plain}
\bibliography{fw}  

\appendix

\section{Discussion of Optimistic Algorithms}\label{app:optimistic}
Here we discuss optimistic algorithms for online learning. Assume that the loss functions are linear and let $\ell_t(\cdot)= \langle l_t , \cdot \rangle$.
When there exists some pattern in the loss sequence $\{ l_t \}$,
the online learning problem may not be that adversarial.
That is, the loss vectors may be predictable. 
If a ``\emph{good guess}'' of $l_t$ is available, then one might hope to have a better bound than $O(\sqrt{T})$. Let $m_t$ be the algorithm's guess for $l_{t}$. If $m_{t}$ is close to $l_{t}$, then the regret should be significantly better than $O(\sqrt{T})$.
There have been several papers that obtain $O(\sqrt{ \sum_{t=1}^T \| l_t -m_t \|^2 } )$ regret in recent years (e.g \cite{CJ12,RK13,SALS15}). 
Among these, perhaps \OFTRL (\cite{SALS15}) has the simplest update rule:
\begin{align}
\OFTRL\quad\quad & x_t = \argmin_{x \in \K}\textstyle \lr{ m_t + \sum_{s=1}^{t-1} l_s }{x} + \frac{1}{\eta} R(x)
\end{align}

\section{Examples of $\beta$-Gauge sets}\label{app:betagauge}
\cite{D15} lists three known classes of strongly convex sets. They are all in the form of norm ball constraints. These are all $\beta$-Gauge sets.
\begin{enumerate}
	\item{$\ell_p$ balls: $\| x \|_p \leq r, \forall p \in (1,2]$. The strong convexity of the set is $\lambda=\frac{p-1}{r}$ and its squared of gauge function is $\frac{1}{r^2} \| x \|_{p}^{2}$, which is a $\frac{2(p-1)}{r^2}$-strongly convex function with respect to norm $\| \cdot \|_p$ by Lemma 4 in \cite{D15}. }
	\item{Schatten $p$ balls: $\| \sigma(X) \|_p \leq r$ for $p \in (1,2]$, where $\sigma(X)$ is the vector consisting of singular values of the matrix $X$. 
		The strong convexity of the set is $\lambda=\frac{p-1}{r}$
		and the squared gauge function is $\frac{1}{r^2} \| \sigma(X) \|_{p}^{2}$, which is a $\frac{2(p-1)}{r^2}$-strongly convex function with respect to norm $\| \sigma(\cdot) \|_p $ by Lemma 6 in \cite{D15}.}
	\item{	Group (s,p) balls: $\| X \|_{s,p}  = \| (\| X_1\|_s, \| X_2\|_s, \dots, \| X_m\|_s)  \|_p \leq r$
		where $X \in \reals^{{m \times n}}$, $X_{j}$ represents the $j$-th row of $X$, and $s,p \in (1,2]$.
		The strong convexity of the set is $\lambda=\frac{2(s-1)(p-1)}{r( (s-1)+(p-1) ) }$ and its squared gauge function is $\frac{1}{r^2} \| X \|_{s,p}^{2}$, which is a $\frac{2(s-1)(p-1)}{r^2( (s-1)+(p-1) ) }$-strongly convex function with respect to norm $\| \cdot\|_{s,p}$ by Lemma 8 in \cite{D15}.}
\end{enumerate}

\section{Proof of Theorem~\ref{thm:convergence}} \label{app:thm:convergence}

\textbf{Theorem~\ref{thm:convergence}} 
\textit{
  Assume that a convex-concave game payoff $g(\cdot, \cdot)$ and a $T$-length sequence $\balpha$ are given. Assume that we run Algorithm~\ref{alg:game} using no-regret procedures $\alg^x$ and $\alg^y$, and the $\balpha$-weighted average regret of each is $\avgregret{x}$ and $\avgregret{y}$, respectively. Then the output  $(\bar{x}_{\balpha},\bar{y}_{\balpha})$ is an $\epsilon$-equilibrium for $g(\cdot, \cdot)$, with
  \[
    \epsilon = \avgregret{x} + \avgregret{y}.
  \] 
} 

\begin{proof}
The proof basically replicates the one in \cite{AW17} 
except that the loss functions of both learners are re-weighted here.
It will turn out that the re-weighting trick is one of the keys to getting fast rate for some games.

In each round, the weighted loss function of the $x$-player is $\alpha_t \ell_t(\cdot) : \XX \to \reals$, where $\ell_t(\cdot) := g(\cdot, y_t)$. The $y$-player, on the other hand, observes her own sequence of loss functions $\alpha_t  h_t(\cdot) : \YY \to \reals$, where $h_t(\cdot) := -  g(x_t, \cdot)$.

\begin{eqnarray}
\frac{1}{\sum_{s=1}^T \alpha_s}   \sum_{t=1}^T \alpha_t g(x_t, y_t) 
  & = & \frac{1}{\sum_{s=1}^T \alpha_s} \sum_{t=1}^T  - \alpha_t  h_t(y_t)  \notag \\
  \text{} \; & = & 
    - \frac{1}{\sum_{s=1}^T \alpha_s} \inf_{y \in \YY} \left\{ \sum_{t=1}^T  \alpha_t  h_t(y) \right\} - \frac{ \text{Regret}_T^y }{  \sum_{s=1}^T \alpha_s } \notag \\
  \; & = &
    \sup_{y \in \YY} \left\{ \frac{1}{\sum_{s=1}^T \alpha_s} \sum_{t=1}^T  \alpha_t g( x_t , y ) \right\} - \avgregret{y}  \notag \\
  \text{(Jensen)} \;  & \geq & 
    \sup_{y \in \YY} g\left({\textstyle \frac{1}{\sum_{s=1}^T \alpha_s} \sum_{t=1}^T  \alpha_t x_t }, y \right)  - \avgregret{y} 
     \label{eq:ylowbound}  \\
  & \geq & \inf_{x \in \XX} \sup_{y \in \YY} g\left( x , y \right) - \avgregret{y}  \notag
\end{eqnarray}

Let us now apply the same argument on the right hand side, where we use the $x$-player's regret guarantee.
\begin{eqnarray}
\frac{1}{\sum_{s=1}^T \alpha_s}  \sum_{t=1}^T  \alpha_t g(x_t, y_t) & = & \frac{1}{\sum_{s=1}^T \alpha_s} \sum_{t=1}^T  \alpha_t \ell_t(x_t) \notag \\
  & = & \inf_{x \in \XX} \left\{ 
    \sum_{t=1}^T \frac{1}{\sum_{s=1}^T \alpha_s} \alpha_t \ell_t(x) \right \} + \frac{ \text{Regret}_T^x }{  \sum_{s=1}^T \alpha_s }  \notag \\
  & = & \inf_{x \in \XX} \left\{  \sum_{t=1}^T \frac{1}{\sum_{s=1}^T \alpha_s} \alpha_t g(x, y_t) \right \} + \avgregret{x} \notag \\
  & \leq & \inf_{x \in \XX}  
    g\left(x,{ \textstyle \sum_{t=1}^T \frac{1}{\sum_{s=1}^T \alpha_s} \alpha_t y_t}\right) + \avgregret{x} 
    \label{eq:xupbound} \\
  & \leq & \sup_{y \in \YY} \inf_{x \in \XX}  g(x,y) + \avgregret{x}  \notag
\end{eqnarray}
Note that the game value $V^{*}= \inf_{x \in \XX} \sup_{y \in \YY} g(x,y) = \sup_{y \in \YY} \inf_{x \in \XX}  g(x,y)$.
Combining (\ref{eq:ylowbound}) and (\ref{eq:xupbound}), we see that:
\begin{align*}
    \sup_{y \in \YY} g\left({\textstyle \bar{x}_{\balpha}}, y \right)  - \avgregret{y}  \le \inf_{x \in \XX}  
g\left(x,{ \textstyle \bar{y}_{\balpha}}\right) + \avgregret{x} 
\end{align*}
which implies that $(\bar{x}_{\balpha}, \bar{y}_{\balpha})$ is an $\avgregret{x} + \avgregret{y} $ approximate solution to the saddle point problem.
\end{proof}

\section{Proof of Lemma~\ref{2in3sums} from Proof of Theorem~\ref{thm:FWgame}} \label{app:2in3sums}
We prove the following stronger statement.
\begin{lemma}\label{2in3sums}	
Let $x_{t}$ be the \OFTL update	$x_t = \argmin_{x \in \XX} \sum_{s=1}^{t-1} \alpha_s \ell_s(x) + m_{t}(x)$,
	where $m_{t}(x) = \alpha_{t} \ell_{{t-1}}(x)$. Let $z_t$ be the standard \FTL update given by $z_t = \argmin_{x} \sum_{s=1}^{t-1} \alpha_s \ell_s(x)$. Let $\alpha_t = t$. Let $\ell_t$ be $(1/L)$-strongly convex for all $t$.

\begin{equation} \label{bound:2sums}
\begin{aligned}
 \sum_{t=1}^{T} m_t(x_t) - m_t(z_{t+1}) + \sum_{t=1}^T \alpha_t \ell_t(z_{t+1}) - \alpha_t \ell_t(w)
 \leq D_T
\end{aligned}
\end{equation}
for any $w \in \XX$, where $$D_T= -\sum_{t=1}^T \pr{\frac{A_{t-1}}{2L} \| z_t - x_t\|^2 
+ \frac{A_{t}}{2L} \| z_{t+1} - x_t\|^2}.$$ 
\end{lemma}

We use the following standard lemma for strongly convex functions.
\begin{lemma} \label{scfunc}
	For a $\sigma$-strongly convex function $F(\cdot)$, suppose $p^*$ is the minimizer. We have $$F(p^*) \leq F(q) + \langle \nabla F(p^*), p^*- q  \rangle - \frac{\sigma}{2} \|  p^* - q \|^2,$$ which leads to
	$$F(p^*) \leq F(q)  - \frac{\sigma}{2} \| p^* - q \|^2.$$
\end{lemma}

\begin{proof}[Proof of Lemma~\ref{2in3sums}]

We use mathematical induction for the proof.
The proof in the following uses Lemma~\ref{scfunc}, which is a property of strongly convex functions.
For the base case $T=0$, it holds because $0\leq0$.

Assume it holds for $T = \tau-1$, that is
\begin{equation}
 \sum_{t=1}^{\tau-1} m_t(x_t) - m_t(z_{t+1}) + \sum_{t=1}^{\tau-1} \alpha_t \ell_t(z_{t+1}) - \alpha_t \ell_t(w)
 \leq D_{\tau-1}
\end{equation}
Then,
\begin{equation}
\begin{aligned}
& \sum_{t=1}^{\tau} m_t(x_t) - m_t(z_{t+1}) + \sum_{t=1}^\tau \alpha_t \ell_t(z_{t+1}) 
\\ & \overset{(a)}{\leq} D_{\tau-1} + m_\tau(x_\tau) - m_\tau(z_{\tau+1})  + \alpha_\tau \ell_\tau(z_{\tau+1}) + \sum_{t=1}^{\tau-1}\alpha_t \ell_t (z_\tau) 
\\ & \overset{(b)}{\leq} D_{\tau-1} - \frac{\sigma_{\tau-1}}{2} \| z_\tau - x_\tau\|^2 + m_\tau(x_\tau) - m_\tau(z_{\tau+1})  + \alpha_\tau \ell_\tau(z_{\tau+1}) + \sum_{t=1}^{\tau-1}\alpha_t \ell_t (x_\tau) 
\\ & = D_{\tau-1} - \frac{\sigma_{\tau-1}}{2} \| z_\tau - x_\tau\|^2  - m_\tau(z_{\tau+1})  + \alpha_\tau \ell_\tau(z_{\tau+1}) + ( \sum_{t=1}^{\tau-1}\alpha_t \ell_t (x_\tau) + m_\tau(x_\tau) )
\\ & \overset{(c)}{\leq} D_{\tau-1} - \frac{\sigma_{\tau-1}}{2} \| z_\tau - x_\tau\|^2 
- \frac{\sigma_{\tau-1}'}{2} \| z_{\tau+1} - x_\tau\|^2 
 - m_\tau(z_{\tau+1})  + \alpha_\tau \ell_\tau(z_{\tau+1}) \\& + ( \sum_{t=1}^{\tau-1}\alpha_t \ell_t (z_{\tau+1}) + m_\tau(z_{\tau+1} ))
\\ &  = D_{\tau-1} - \frac{\sigma_{\tau-1}}{2} \| z_\tau - x_\tau\|^2 
- \frac{\sigma_{\tau-1}'}{2} \| z_{\tau+1} - x_\tau\|^2 
  + \sum_{t=1}^{\tau}\alpha_t \ell_t (z_{\tau+1}) 
\\ & \leq D_\tau + \sum_{t=1}^{\tau}\alpha_t \ell_t (w)
\end{aligned}
\end{equation}
where $(a)$ is by the assumption that it holds at $\tau-1$ and by setting $w=z_{\tau}$,
$(b)$ is by using the strong convexity lemma (Lemma~\ref{scfunc}) to get the following,
\begin{equation}
 \sum_{t=1}^{\tau-1}\alpha_t \ell_t (z_\tau) \leq  \sum_{t=1}^{\tau-1}\alpha_t \ell_t (x_\tau) - \frac{\sigma_{\tau-1}}{2} \| z_\tau - x_\tau\|^2,
\end{equation}
as $z_\tau$ is the minimizer of the $\sigma_{\tau-1}$-strongly convex function
$\sum_{t=1}^{\tau-1}\alpha_t \ell_t (\cdot) $ for $\sigma_{\tau-1} := \sum_{s=1}^{\tau-1}\frac{\alpha_s}{L} = \frac{A_{\tau-1}}{L}$,
$(c)$ is by using Lemma~\ref{scfunc} again to get
\begin{equation}
 \sum_{t=1}^{\tau-1} \alpha_t \ell_t (x_{\tau}) + m_\tau(x_\tau) \leq  \sum_{s=1}^{\tau-1}\alpha_t \ell_t (z_{\tau+1}) +m_\tau(z_{\tau+1}) - \frac{\sigma_{\tau-1}'}{2} \| z_{\tau+1} - x_\tau\|^2,
\end{equation}
as $x_\tau$ is the minimizer of the $\sigma_{\tau-1}'$-strongly convex function
$\sum_{t=1}^{\tau-1}\alpha_t \ell_t (\cdot) + m_{\tau}(\cdot)$ for $\sigma_{\tau-1}':= \sum_{s=1}^{\tau-1}\frac{\alpha_s}{L} + \frac{\alpha_\tau}{L}= \frac{A_{\tau}}{L}$,
and the last inequality is by the fact that $z_{\tau+1}$ is the minimizer of
$\sum_{t=1}^{\tau}\alpha_t \ell_t (\cdot) $.
\end{proof}

\section{Optimistic FTRL} \label{app:optFTRL}
In this subsection, we analyze the regret bound for \OFTRL.
Let $\theta_{t}$ be the loss function in round $t$ and let the cumulative loss vector be
$L_{t} = \sum_{s=1}^t \theta_{s}$.
The update of \OFTRL is
\begin{equation} \label{optFTRL}
y_t  = \arg \min_{y \in \YY} \langle y , L_{t-1} + m_t \rangle + \frac{1}{\eta} R(y),
\end{equation}
where 
$m_{t}$ is the learner's guess of the loss vector $\theta_{t}$,
$R(\cdot)$ is a $\beta$-strong convex function with respect to a norm ($\| \cdot \|$) on the constraint set $\YY$
and $\eta$ is a parameter. 
The analysis of \cite{SALS15} is based on the assumption that $\YY$ is a simplex.
Here, we want to extend their results for $\YY$ being an arbitrary convex set.
Define the regret to be $\text{Regret}:= \sum_{t=1}^{T} \langle y_t - y^{*}, \theta_{t} \rangle$.
The proof basically replicates the one of (\cite{HaipengLuo17}) except that we consider any convex set and any strongly convex regularization term instead of negative entropy as in his note.

\begin{theorem} \label{thm:oftrl}
\OFTRL (\ref{optFTRL}) has regret
\begin{equation} \label{opt-ftrl}
\begin{aligned}
\textstyle \text{Regret}:= \sum_{t=1}^{T} \langle y_t - y^{*}, \theta_{t} \rangle 
\leq & \textstyle \frac{  R(y^*) - R(z)  - \frac{\beta}{2} (\sum_{{t=1}}^{T} \| y_t - z_t \|^2 + \sum_{{t=1}}^{T} \| y_{t} - z_{t+1} \|^2)}{\eta} 
\\& \textstyle +\sum_{t=1}^{T} \frac{\eta}{\beta} \| \theta_t - m_{t} \|^2_* ,
\end{aligned}
\end{equation}
where $z_{t}= \arg\min_{y \in \YY} \langle y, L_{{t-1}} \rangle + \frac{1}{\eta} R(y)$ is the update of the standard \FTRL,
and $z = \arg\min_{y \in \YY} R(y)$.
\end{theorem}

\begin{proof}
We can re-write the regret as
\begin{equation}
\begin{aligned}
\text{Regret}:= \sum_{t=1}^{T} \langle y_t - y^{*}, \theta_{t} \rangle
= \sum_{t=1}^{T} \langle y_t - z_{t+1} , \theta_t - m_t \rangle + 
 \sum_{t=1}^{T}  \langle y_t - z_{t+1} , m_t \rangle + \langle z_{t+1} - y^*, \theta_t \rangle 
\end{aligned}
\end{equation}
Let us analyze the first sum
\begin{equation}
\sum_{t=1}^{T} \langle y_t - z_{t+1} , \theta_t - m_t \rangle.
\end{equation}
Now using Lemma~\ref{aux:optFTRL}
with $y_{1}=y_{t}$, $u_{1}=\sum_{{s=1}}^{{t-1}} \theta_{s} + m_{t}$ and $y_{2}=z_{t+1}$, 
$u_{2}=\sum_{{s=1}}^{{t}} \theta_{s} $ in the lemma,
we have 
\begin{equation} \label{qq1}
\sum_{t=1}^{T} \langle y_t - z_{t+1} , \theta_t - m_t \rangle
\leq \sum_{t=1}^{T} \|  y_t - z_{t+1} \| \| \theta_t - m_t\|_* 
\leq 
\sum_{t=1}^{T} \frac{\eta}{\beta} \| \theta_t - m_t \|^2_*.
\end{equation}
For the other sum,
\begin{equation}
\sum_{t=1}^{T}  \langle y_t - z_{t+1} , m_t \rangle + \langle z_{t+1} - y^*, \theta_t \rangle,
\end{equation}
we are going to show that it is upper-bounded by
\begin{equation}
\frac{ R(y^*) - R(z) - D_T }{\eta}, 
\end{equation}
for any $y^* \in \YY$, where $D_{T} = \sum_{{t=1}}^{T} \frac{\beta}{2} \| y_t - z_{t} \|^2 
+ \frac{\beta}{2} \| y_t - z_{t+1} \|^2$.

Using induction, we see that 
the base case $T=0$ clearly is true.
Assume that it also holds for $T-1$, for a $T \geq 1$.
Then, we have
\begin{equation}
\begin{aligned}
& \sum_{t=1}^{T}  \langle y_t - z_{t+1} , m_t \rangle + \langle z_{t+1} , \theta_t \rangle
\\& \leq \langle y_T - z_{T+1}, m_T \rangle + \langle z_{T+1}, \theta_T \rangle
+ \frac{ R(z_T) - R(z) - D_{T-1} }{\eta} + \langle z_{T}, L_{T-1} \rangle
\\& \overset{(a)}{\leq}
\langle y_T - z_{T+1}, m_T \rangle + \langle z_{T+1}, \theta_T \rangle
+ \frac{ R(y_T) - R(z) - D_{T-1} - \frac{\beta}{2} \| y_T - z_{T} \|^2 }{\eta} 
\\ & + \langle y_T, L_{T-1} \rangle
\\& = \langle z_{T+1}, \theta_T - m_T \rangle 
+ \frac{ R(y_T) - R(z) - D_{T-1} - \frac{\beta}{2} \| y_T - z_{T} \|^2 }{\eta} + \langle y_T, L_{T-1} + m_T \rangle
\\& \overset{(b)}{\leq}  \langle z_{T+1}, \theta_T - m_T \rangle 
+ \frac{ R(z_{T+1})- R(z) - D_{T-1} - \frac{\beta}{2} \| y_T - z_{T} \|^2 
- \frac{\beta}{2} \| y_T - z_{T+1} \|^2
}{\eta}
\\&  + \langle z_{T+1}, L_{T-1} + m_T \rangle
\\& = \langle z_{T+1}, L_T \rangle 
+ \frac{  R(z_{T+1})- R(z) - D_{T} }{\eta}
\\ & \overset{(c)}{\leq} 
\langle y^*, L_T \rangle 
+ \frac{  R(y^*) - R(z) - D_{T} }{\eta},
\end{aligned}
\end{equation}
where (a) and (b) are by strong convexity so that
\begin{equation}
\langle z_{T}, L_{T-1} \rangle + \frac{R(z_T)}{\eta}
\leq \langle y_{T}, L_{T-1} \rangle + \frac{R(y_T)}{\eta} - \frac{\beta}{2 \eta} \| y_T - z_T  \|^2,
\end{equation}
and
\begin{equation}
\langle y_{T}, L_{T-1} + m_T \rangle + \frac{R(y_T)}{\eta}
\leq \langle z_{T+1}, L_{T-1} + m_T \rangle + \frac{R(z_{T+1})}{\eta} - \frac{\beta}{2 \eta} \| y_T - z_{T+1}  \|^2.
\end{equation}
(c) is because $z_{T+1}$ is the optimal point of 
$\arg\min_y \langle y , L_T \rangle + \frac{R(y)}{\eta}$.

Now we have shown that 
\begin{equation} \label{q1}
\sum_{t=1}^{T}  \langle y_t - z_{t+1} , m_t \rangle + \langle z_{t+1} - y^*, \theta_t \rangle
\leq \frac{ R(y^*) - R(z) - D_{T} }{\eta},
\end{equation}
for any $y^* \in \YY$.
%

Combining (\ref{qq1}) and (\ref{q1}) completes the theorem (\ref{opt-ftrl}).

\end{proof}

\begin{lemma} \label{aux:optFTRL}
\textit{
Denote $y_1 = \arg\min_{y} \langle y, u_1 \rangle + \frac{1}{\eta} R(y)$
and $y_{2} = \arg\min_{y} \langle y, u_2 \rangle + \frac{1}{\eta} R(y)$
for a $\beta$-strongly convex function $R(\cdot)$ with respect to a norm $\| \cdot \|$.
We have $\| y_{1} - y_{2} \| \leq \frac{\eta}{\beta} \| u_1 - u_2\|_{*}$.} 
\end{lemma}

\begin{proof}
From strong convexity, we have
$$ \langle y_1, u_1 \rangle + \frac{1}{\eta} R(y_1)
\leq \langle y_2, u_1 \rangle + \frac{1}{\eta} R(y_2) - \frac{\beta}{2 \eta} \| y_1 - y_2\|^2,$$
and
$$ \langle y_2, u_2 \rangle + \frac{1}{\eta} R(y_2)
\leq \langle y_1, u_2 \rangle + \frac{1}{\eta} R(y_1) - \frac{\beta}{2 \eta} \| y_1 - y_2\|^2.$$
Adding the above inequalities together,
we get 
$$\langle y_2 - y_1, u_2 - u_1 \rangle \leq - \frac{\beta}{\eta} \| y_1 - y_2\|^2,$$ 
or
$$\langle y_1 - y_2, u_2 - u_1 \rangle \geq \frac{\beta}{\eta} \| y_1 - y_2\|^2.$$
So, by using H\"older's inequality,
$$  \frac{\eta}{\beta} \|  u_1 - u_2\|_* \geq \| y_1 - y_2 \|.$$
\end{proof}

\section{Proof of Corollary~\ref{cor:FWcor}}\label{app:fwcor}
\begin{proof}[Proof of Corollary~\ref{cor:FWcor}]
	Let $\epsilon = O\pr{ \frac{ L(1+\frac{L}{\sigma}) }{\beta  T^2}}$. We can see from \Cref{eq:gauge_FTRL} that the $y$ iterates are exactly the actions of the $y$-player in \Cref{thm:FWgame} when $R(y) = \gamma_{\YY}^2(y)$. Likewise, the $x$-player's iterates are exactly those of the $x$-player in \Cref{thm:FWgame}, as shown in Equation~\ref{xt}. As the weight vector $\balpha$ is the same as in \Cref{thm:FWgame}, we know that $(\bar{x}_{\balpha},\bar{y}_{\balpha})$ is an $\epsilon$-approximate saddle point to the FW game $g(x,y) = f^*(x) - \lr{x}{\bar{y}_{\balpha}}$. Then we have:
	\begin{align*}
	f(\bar{y}_{\balpha}) = \sup_{x\in \XX} \lr{x}{\bar{y}_{\balpha}} - f^*(x) = -\inf_{x\in\XX} g(x,\bar{y}_{\balpha}) \le -V^* + \epsilon 
	\end{align*}
	where $V^*$ is the value of the FW game, which we showed in \Cref{eq:fwgamevalue} is equal to $-\inf_{y\in \YY} f(y)$.
	
	Also note that \Cref{alg:FW} does not need to store all the points \{$y_{t}\}$, as it can instead maintain a variable recording the weighted sum of the points.	
\end{proof}

\section{The $y$-player plays \BTPL in the FW game} \label{app:BTPL}

Suppose now the $y$-player plays in round $t$ plays \BTPL,
\begin{equation} \label{pBPTL}
y_t = E_{\xi \sim B_u(\cdot)}[ \arg\min_{y\in \YY} \langle y, \sum_{s=1}^t x_s + \xi \rangle],
\end{equation}
where $B_u(\cdot)$ is a uniform sample on the $\ell_2$ ball of radius $u$.
We are going to analyze the distance term $\| y_t - y_{t-1} \|^2$
in the regret of the $x$-player (\ref{regret_x}).
Define $\Psi(x) = E_{\xi \sim B_u(\cdot)}[ \max_{y \in \YY} \langle y, x \rangle ]$.
If the max expression inside the expectation has a unique maximizer 
we can swap the expectation and gradient (Proposition 2.2 in (\cite{B73})).
So,
the gradient can be written as $\nabla \Psi(x) = E_{\xi \sim B_u(\cdot)}[ \arg\max_{y \in \YY} \langle y, x \rangle ]$.
Now we see that the update of Be-the-Perturbed-Leader (\ref{pBPTL}) can be connected to the gradient of $\Psi(\cdot)$.
That is,
\begin{equation}
y_t = \nabla \Psi(- \sum_{s=1}^t x_s) = E_{\xi \sim B_u(\cdot)}[ \arg\max_{y \in \YY} \langle y, - \sum_{s=1}^t x_s + \xi \rangle ].
\end{equation}
Using Lemma E.2 in (\cite{DBW12}),
we know that $\Phi$ is a $\frac{ L_0 \sqrt{d} }{u}$ smooth function, where $L_{0}$ is the bound of the norm $\| \nabla \Psi(\cdot)  \|$ and $d$ is the dimension of $\YY$.

Consequently.
\begin{equation} \label{shrinkP1}
\begin{aligned}
& \| y_t - y_{t-1}\|^2 = \norm{ \nabla \Psi\pr{- \sum_{s=1}^t x_s} -\nabla \Psi \pr{- \sum_{s=1}^{t-1} x_s} }^2
\leq \frac{L_0^2 d}{u^2} \norm{ - \sum_{s=1}^t x_s + \sum_{s=1}^{t-1} x_s }^2
\leq \frac{L L_0^2 d}{u^2},
\end{aligned}
\end{equation}
where the last inequality is by the fact that $x_t$ is a gradient vector and that the norm of gradient is bounded by the smooth constant $L$ by the assumption.
One can also show that regret of Be-the-perturbed-leader is bounded by $O(u)$ (e.g. \cite{KV05,cesa2006prediction}).

Combining $\regret{x}$ of (\ref{regret_x}) and $\regret{y}$,
\begin{equation}
\begin{aligned}
&\frac{ \regret{x} }{ A_T} + \frac{ \regret{y} }{ A_T}
\\ & \leq  \frac{\sum_{t=1}^T
\alpha_t  ( (\frac{\alpha_t L}{A_t})  \| y_t - y_{t-1} \|^2 + \frac{1}{\sigma} (\frac{\alpha_t L }{A_t})^2  \| y_t - y_{t-1} \|^2 ) + O(u)
 }{ \sum_{t=1}^T \alpha_t}
\\ & \leq  O\pr{ \frac{ (\sum_{t=1}^T \frac{L^2 L_0^2 d}{ \sigma u^2} ) + u }{T^2} },
\end{aligned}
\end{equation}
where the last inequality is due to that $\alpha_t = t$ and so that $A_T = O(T^2$). Setting $u=O(T^{{1/3}})$ leads to 
$\frac{ \regret{x} }{ A_T} + \frac{ \regret{y} }{ A_T}  =O(\frac{L^2 L_0^2 d}{\sigma T^{5/3}})$.
The rate is better than $O(1/T)$ when the dimension $d$ is small.
Note that we do not make any assumption on the constraint set $\YY$, so the rate holds for any convex set $\YY$.
Yet, to get the fast rate, one needs to have a efficient way to compute (\ref{pBPTL}),
which seems to involve large number of calls to the linear oracle in each iteration.
We leave it as an open problem.

\section{The $y$-player plays \BR in the FW game} \label{app:BR}

If the constraint set is strongly convex, $\| y_t - y_{t-1}  \|^2$ is indeed shrinking even if the $y$-player just plays $\texttt{BestResponse},$
assuming the gradient is lower bounded by a constant $B$.
To see this, let us use Lemma~\ref{lm:lip} below, 
\begin{equation} \label{shrink}
\begin{aligned}
\| y_t - y_{t-1}\| & \leq \frac{ 2 \| x_t - x_{t-1} \|}{ \lambda ( \| x_t\| + \| x_{t-1}\| ) }
= \frac{ 2 \| \nabla f( \bar{y}_{t-1}) - \nabla f( \bar{y}_{t-2}) \|}{ \lambda ( \| \nabla f( \bar{y}_{t-1})\| + \| \nabla f( \bar{y}_{t-2})\| ) }
\\ & \leq 
\frac{ 2 L \| \bar{y}_{t-1} - \bar{y}_{t-2} \|}{ \lambda ( \| \nabla f( \bar{y}_{t-1})\| + \| \nabla f( \bar{y}_{t-2})\| ) }
= O( \frac{L}{\lambda t B}),
\end{aligned}
\end{equation}
where the first inequality is due to the lemma and the last equality is by assuming 
$\| \nabla f( \cdot)\| \geq B$.

Since the $y$-player plays \BR, its regret is $0$.
Combining this with $\regret{x}$ of (\ref{regret_x}) and using (\ref{shrink})
we have 
\begin{equation}
\begin{aligned}
&\frac{ \regret{x} }{ A_T} + \frac{ \regret{y} }{ A_T}
\\ & \leq  \frac{\sum_{t=1}^T
\alpha_t  ( (\frac{\alpha_t L}{A_t})  \| y_t - y_{t-1} \|^2 + \frac{1}{\sigma} (\frac{\alpha_t L }{A_t})^2  \| y_t - y_{t-1} \|^2 ) 
 }{ \sum_{t=1}^T \alpha_t}
\\ & =  O( \frac{L \log T }{\lambda B T^2}  ).
\end{aligned}
\end{equation}

\begin{lemma}  \label{lm:lip}
\footnote{\cite{P96} has discussed the smoothness of the support function on strongly convex sets.
Here, we state the more general result and give a new proof.}

Denote  
$x_p = \argmax_{x \in \K} \langle p, x\rangle $ and $x_q = \argmax_{x \in \K} \langle q, x\rangle $, where $p,q \in \reals^d$ are any nonzero vectors.  
If a compact set $\K$ is a $\lambda$-strongly convex set,
then 
\begin{equation}
    \| x_p - x_q \| \leq \frac{2 \norm{p - q}}{\lambda ( \| p \| + \|q \| )}.
\end{equation}
\end{lemma}

\begin{proof}
Recall that a strongly convex set $\K$ can be written as intersection of some Euclidean balls (c.f. definition 3 in Section~\ref{Pre}).
Namely,
    \[ \K = \underset{u \in \mathcal{S}}{\cap} B_{\frac{1}{\lambda}} \left( x_u - \frac{u}{\lambda} \right) .\]

    Let  ${x_p = \argmax_{x \in \K} \langle \frac{p}{\norm{p}}, x\rangle }$ and ${x_q = \argmax_{x \in \K} \langle \frac{q}{\norm{q}}, x \rangle}$.
    Based on the definition of strongly convex sets, we can see that
    $x_q \in B_{ \frac{1}{\lambda}  } ( x_p - \frac{p}{\lambda \norm{p}})$ and $x_p \in B_{\frac{1}{\lambda}  } ( x_q - \frac{q}{\lambda \norm{q}} )$.
    Therefore, 
    \[
        \| x_q - x_p -  \frac{p}{\lambda \norm{p}} \|^2 \leq \frac{1}{\lambda^2},
    \]
    which leads to
    \begin{equation}
        \label{eqn:ineqSumP1}
     \|p \| \cdot  \| x_p - x_q \|^2 \leq \frac{2}{\lambda} \langle x_p - x_q,  p \rangle.
    \end{equation}
    Similarly,
    \[
        \| x_p - x_q - \frac{q}{\lambda \|q\|} \|^2 \leq \frac{1}{\lambda^2}, 
    \]
    which results in
    \begin{equation}
        \label{eqn:ineqSumP2}
      \| q \| \cdot  \| x_p - x_q \|^2 \leq \frac{2}{\lambda} \langle x_q - x_p,  q \rangle.
    \end{equation}
    Summing (\ref{eqn:ineqSumP1}) and (\ref{eqn:ineqSumP2}), one gets
    $(\|p\| + \| q \|) \| x_p - x_q \|^2 \leq \frac{2}{\lambda} \langle x_p - x_q, p-q \rangle$.
    Applying the Cauchy-Schwarz inequality completes the proof.
\end{proof}

\section{} \label{app:smoothH-1}

\begin{proposition} \label{sameGrad}
For arbitrary $x$, let $\ell(\cdot) := g(\cdot, \yx{x})$. Then $\nabla_{x} \ell(\cdot) \in \partial_{x} s(\cdot)$.
\end{proposition}
\begin{proof}
Consider any point $w \in \XX$,
\begin{equation}
\begin{aligned}
 s(w) - s(x) & = g(w,\yx{w}) - g(x, \yx{x}) 
\\ &= g(w,\yx{w}) - g(w, \yx{x})  + g(w,\yx{x}) - g(x, \yx{x})
\geq 0 + g(w,\yx{x}) - g(x, \yx{x})
\\ & \geq  \langle \partial_{x} g( \cdot, \yx{x}) , w - x \rangle = \langle \nabla_{x} \ell(\cdot) , w - x \rangle
\end{aligned}
\end{equation}
where the first inequality is because that $\yx{w}$ is the best response to $w$, the second inequality is due to the convexity.
The overall statement implies that $\nabla_{x} \ell_t(\cdot)$ is a subgradient of $s$ at $x$.
\end{proof}

We assume that
\begin{equation} \label{payoff-1}
g( x, y ) = a(x) +  x^\top M y - b(y),
\end{equation}
where $a(x)$ is any $\sigma_x$-strongly convex function of $x$, $M$ is a matrix, and $b(y)$ is any $\sigma_y$-strongly convex function of $y$ \footnote{We note that some machine learning applications can be written into this form (see, e.g \cite{XNLS05,ZX15}).}.

\begin{proposition} \label{smoothH-1}
Assume the payoff function is in the form of (\ref{payoff-1}) and that the norm is $\ell_2$ norm (i.e. $\|\cdot \| = \|\cdot \|_2$).
Given any points $w$, $z \in \XX$, the gradient of $s(\cdot)$ satisfies
$\| \nabla_w s(w) - \nabla_z s(z)\| \leq (\frac{\|M\|^2}{\sigma_y} + L) \| w - z\|$.
This shows that $s(\cdot)$ is a $(\frac{\|M\|^2}{\sigma_y} + L)$-smooth function.
\end{proposition}

\begin{proof}

\begin{equation} \label{aux:h}
\begin{aligned}
& \| \nabla_w s(\cdot) - \nabla_z s(\cdot )\| 
\overset{(a)}{=} \|  \nabla_w g( \cdot, \yx{w}) - \nabla_z g( \cdot, \yx{z})  \| 
\\ & \overset{(b)}{\leq}
\|  \nabla_w g( \cdot, \yx{w}) - \nabla_w g( \cdot, \yx{z})  \| 
+ 
\|  \nabla_w g( \cdot, \yx{z}) - \nabla_z g( \cdot, \yx{z})  \|  
\\ & \overset{(c)}{\leq}
\|  \nabla_w g( \cdot, \yx{w}) - \nabla_w g( \cdot, \yx{z})  \| 
+ 
L \| w - z \| 
\\ & =
 \| M( \yx{w}  - \yx{z} ) \| 
+ 
L \| w - z \|
\leq \| M \|\cdot \| \yx{w}  - \yx{z}  \| + L \| w - z \|
\end{aligned}
\end{equation}
where $(a)$ is because of the gradient coincidence (Proposition~\ref{sameGrad}),
$(b)$ is due to triangle inequality, and $(c)$ is the result of the $L$-smoothness of $g(\cdot,\cdot)$,
and the last inequality is a standard matrix inequality.

Now we are going to bound 
$ \| \yx{w} - \yx{z} \|.$ 
Recall that the $y$-player plays \BR, so
$\yx{w} = \arg\min_{y \in \YY} - \langle M y, w \rangle + b(y)$ 
and $\yx{z} = \arg\min_{y \in \YY} - \langle M y, z \rangle + b(y)$.
Since $b(\cdot)$ is a $\sigma_{y}$-strongly convex function with respect to $\| \cdot\|_{2}$,
its conjugate $b^{*}(\cdot)$ is $\frac{1}{\sigma_y}$-strongly smooth with respect to $\| \cdot \|_{2}$.

In the following,
denote function  $\psi_w(y)= - \langle M y, w \rangle + b(y)$
so that $\yx{w} = \arg\min_y \psi_w(y)$ and $\yx{z} = \arg\min_y \psi_w(y) + \langle M^\top(w-z), y \rangle
:= \arg\min_y \psi_{z}(y)$.
By definition of Fenchel-conjugate (c.f. \cite{R96,BL06}),
\begin{equation}
\psi_w^*(0) = \max_{y} \langle 0, y \rangle  - \psi_w(y) = - \min_y \psi_w(y) = - \psi_w(\yx{w}),
\end{equation}
and
\begin{equation}
\nabla \psi_w^*(0) = \yx{w}.
\end{equation}
Similarly,
\begin{equation}
\psi_w^*( M^\top(z-w)) = \max_{y} \langle M^\top (z-w), y \rangle  - \psi_w(y) = - \min_y (  
\langle M^\top(w-z), y \rangle + \psi_w(y) = - \psi_z(\yx{z}),
\end{equation}
and
\begin{equation}
\nabla \psi_w^*( M^\top(z-w)) = \yx{z}.
\end{equation}
Therefore,
\begin{equation}
\| \yx{w} - \yx{z} \| = \| \nabla \psi^*(0) - \nabla \psi^*( M^\top(z-w)) \| \leq \frac{1}{\sigma_b}  
\| 0 - M^\top(z-w)\|  = \frac{\| M\|}{\sigma_b} \| w- z \|.
\end{equation}
Combining the above results, we have 
\begin{equation}
\begin{aligned}
& \| \nabla_w s(\cdot) - \nabla_z s(\cdot ) \| \leq 
\| M\| \|  \yx{w}  - \yx{z}  \| 
+ 
L \| w - z \|  \leq \pr{ \frac{ \| M \|}{\sigma_b} + L} \| w - z\|.
\end{aligned}
\end{equation}

\end{proof}

\section{} \label{app:smooth2}

\begin{proposition} \label{thm:smooth2}
Assume that for all $x\in \XX$, $\arg \min_{y\in \YY} g(x,y) \in \YY$. Also assume that the function $g(x,y)$ is $\sigma_y$-strongly concave in $y$ and that the following smoothness condition holds for any given points $w,z \in \XX$,
$  \| \nabla_{\yx{z}} g(w,\cdot) - \nabla_{\yx{z}} g(z,\cdot)   \| \le L \| w - z \|.$
Then, $s(\cdot)$ is a $L(1+\frac{2L}{\sigma_y})$-smooth function.
\end{proposition}

\begin{proof}
Following the same proof in (\ref{aux:h}), we have
\begin{equation} \label{aux:ss}
\begin{aligned}
& \| \nabla_w s(\cdot) - \nabla_z s(\cdot )\| 
\leq L \|  \yx{w}  - \yx{z}  \| 
+ 
L \| w - z \|.
\end{aligned}
\end{equation}

Now we need to bound $\| \yx{w} - \yx{z} \|$. 
Note that $-g(x,\cdot)$ is a $\sigma_{y}$-strongly convex function.
  Therefore, we have 
\begin{equation} \label{sm1}
( - g(w,\yx{z}) ) -  ( - g(w,\yx{w}) )  \ge \langle -\nabla_{\yx{w}} g(w,\cdot), \yx{z} - \yx{w} \rangle + \frac{\sigma_y}{2}\| \yx{w} - \yx{z} \|^2 \geq \frac{\sigma_y}{2} \| \yx{w} - \yx{z}\|^2 . 
\end{equation} 
This means that 
$$ g(w,\yx{w}) - g(w,\yx{z}) \geq \frac{\sigma_y}{2} \| \yx{w} - \yx{z}\|^2.$$
From this, we can also conclude that 
\footnote{ $ \sqrt{ \frac{2}{\sigma_y} ( g(w,\yx{w}) - g(w,\yx{z})   } \| \nabla_{\yx{z}} g(w, \cdot) \|
\geq \|  \yx{w} - \yx{z} \|_* \| \nabla_{\yx{z}} g(w, \cdot) \| \geq \langle \yx{z} - \yx{w} , -\nabla_{\yx{z}} g(w, \cdot) \rangle \geq ( - g(w,\yx{z}) ) -  ( - g(w,\yx{w}) ).$} 
\begin{equation} \label{sm2}
\| \nabla_{\yx{z}} g(w, \cdot)  \|^2 \geq \frac{\sigma_y}{2} ( g(w,\yx{w})-g(w,\yx{z}) ).
\end{equation}

  Since $y$ is unconstrained, $\nabla_{\yx{z}} g(z,\cdot) = 0$, from 
the assumption that $\| \nabla_{\yx{z}} g(w,\cdot) - \nabla_{\yx{z}} g(z,\cdot) \| \le L \| w - z \|$,
we have
  \begin{align} \label{sm3}
  \| \nabla_{\yx{z}} g(w,\cdot) \| \le L \| w - z\|.
  \end{align}
  Putting it all together, we get:
  \begin{equation}
  \begin{aligned}
&  \frac{\sigma_y}{2} \| \yx{w} - \yx{z} \|^2  \overset{(\ref{sm1})}{\le} g(w,\yx{w}) -g(w,\yx{z})  
 \overset{ (\ref{sm2}) }{\leq} \frac{2}{\sigma_y} \| \nabla_{\yx{z}} g(w,\cdot)\|^2 
\overset{ (\ref{sm3}) }{\leq} \frac{2L^2}{\sigma_y} \| w - z\|^2\\
  \end{aligned}
  \end{equation}
  So, $\| y_w - y_z\|\leq \frac{2 L}{\sigma_y} \| w - z\|$.
  Substituting it back to (\ref{aux:ss}) completes the proof.
\end{proof}

\section{Proof of \Cref{thm:linear-rate}} \label{app:linear-rate}

\textbf{Theorem~\ref{thm:linear-rate}}
\textit{ 
Assume $g$ satisfies any conditions such that the resulting $s(\cdot)$ is $L$-smooth and $\sigma$-strongly convex function for constants $L$ and $\sigma$ and moreover that $(\arg\min_{x\in \reals^d} s(x)) \in \XX$. Suppose we instantiate Algorithm~\ref{alg:game}, where $x$-player uses SC-AdaGrad (\Cref{alg:SC-AdaNGD}) for loss functions $\alpha_t \ell_t(\cdot) = \alpha_t g(\cdot, y_t)$, $y$-player uses \BR, and the sequence $\balpha$ is defined as $\alpha_t := \| \nabla \ell_t(x_t) \|^{-2}$. Let $T \geq \frac L \sigma \log 3$, and let $G$ be a constant such that $\|\nabla \ell_t\| \le G$ for all $t$. Then the output $(\bar x_{\balpha}, \bar y_{\balpha})$ is an $\epsilon$-equilibrium of $g$ where $\epsilon = O\pr{\frac{G^2 T}{L}e^{-\frac{\sigma}{L}T}}$.
}

The proof of \Cref{thm:linear-rate} is a slight modification of the proof of Theorem 3.2 in \cite{L17}, combined with Proposition~\ref{sameGrad}. As before, we let $\yx{x} = \arg\max_{y\in \YY} g(x,y)$. Recall that $\ell_t(\cdot) = g(\cdot, \yx{x_t})$ and note $\ell_t(x_t) = g(x_t,\yx{x_t}) = s(x_t)$.
Before proving the theorem, let us introduce some lemmas.

\begin{lemma} \label{lem:GSmooth}
For any $L$-smooth convex function $\ell(\cdot):\reals^d \mapsto \reals$, if $x^* =\argmin_{x\in \reals^d} \ell(x)$, then 
$$ \| \nabla \ell(x)\|^2 \le 2 L \left( \ell(x) - \ell(x^*)\right), \quad \forall x\in \reals^d~.$$
\end{lemma} 

\begin{lemma} (\cite{L17}) \label{lem:Log_sum} 
For any non-negative real numbers $a_1,\ldots, a_n\geq 1$,
\begin{align*}
\sum_{i=1}^n \frac{a_i}{\sum_{j=1}^i a_j} 
\le 
1+\log\left( \sum_{i=1}^n a_i\right) ~.
\end{align*}
\end{lemma}

\begin{lemma} (\cite{hazan2007logarithmic,L17}) \label{osc}
Assume the loss function $\ell_t(\cdot): \K\mapsto \reals$ is $\theta_t$-strongly convex. Setting $\eta_t = \pr{\sum_{s=1}^t \theta_s}^{{-1}}$ in Algorithm~\ref{alg:SC-AdaNGD}, we get the following regret for any $x \in \K$:
  \begin{align}
  \sum_{t=1}^T \ell_t(x_t) - \sum_{t=1}^T \ell_t(x) \le \frac{1}{2}\sum_{t=1}^T \eta_t \| \nabla \ell_t(x_t) \|^2
  \end{align}
  where $x_{t+1} = x_t - \eta_t  \nabla \ell_t(x_t)$.
\end{lemma}

\begin{algorithm}[h] 
  \caption{Strongly-Convex Adaptive Gradient Descent (SC-AdaGrad) (\cite{L17})}
  \label{alg:SC-AdaNGD}
  \begin{algorithmic}[1]
    \FOR{$t= 1, 2, \dots, T$}
    \STATE Play $x_t \in \K$
    \STATE Receive a $\theta_t$-strongly convex loss function $\ell_{t}(\cdot)$
    \STATE Set $\eta_t = \frac{1}{\sum_{t=1}^T \theta_t}$
    \STATE Update $x_{t+1} = \Pi_{\XX} (x_t - \eta_t \nabla \ell_t(x_t))$  
    \ENDFOR
  \end{algorithmic}
\end{algorithm}

\begin{lemma} 
Under the conditions in \Cref{thm:linear-rate},
$G^2 \sum_{t=1}^T \norme{\nabla \ell_t(x_t)}^{-2} \ge \frac{1}{3}e^{\frac{\sigma}{L}T}$
\end{lemma}

\begin{proof}
\begin{align}
T &= \sum_{t=1}^T \frac{\norme{ \nabla \ell_t(x_t)}^2}{\norme{ \nabla \ell_t(x_t)}^2}
\overset{Proposition~\ref{sameGrad}}{ = }\sum_{t=1}^T \frac{\norme{\nabla s(x_t)}^2}{\norme{ \nabla \ell_t(x_t)}^2}
\overset{Lemma~\ref{lem:GSmooth}}{ \le} \sum_{t=1}^T \frac{2L}{\norme{ \nabla \ell_t(x_t)}^2} (s(x_t) - s(x^*))\\
&\le \sum_{t=1}^T \frac{2L}{\norme{ \nabla \ell_t(x_t)}^2} (\ell_t(x_t) - \ell_t(x^*))\label{eq:normal},
\end{align}
where the last inequality follows from the fact that $s(x_t) = \ell_t(x_t)$ and $\ell_t(x) = g(x,y_t) \le g(x,\yx{x}) = s(x)$ for any $x$.

Now we define the following function, which is $\frac{\sigma}{\|\nabla \ell_t(x_t)\|^{2}}$-strongly convex.
\begin{align} \label{tf}
\tilde{\ell}_t(x) = \frac{1}{\| \nabla \ell_t(x_t)\|^{2} }\nabla \ell_t(x_t)^T x + \frac{\sigma}{2\norme{\nabla \ell_t(x_t)}^{2}} \| x - x_t\|^2
\end{align}
Then we have
\begin{equation} \label{tmyt}
\begin{aligned}
& T \leq \sum_{t=1}^T \frac{2L}{\norme{ \nabla \ell_t(x_t)}^2} (\ell_t(x_t) - \ell_t(x^*))
\overset{(a)}{\le} \sum_{t=1}^T \frac{2 L}{\norme{ \nabla \ell_t(x_t)}^2} (\nabla \ell_t(x_t)^T (x_t - x^*) - \frac{\sigma}{2} \| x_t - x^*\|^2)\\
&= 2L \sum_{t=1}^T \pr{\tilde{\ell}_t(x_t) - \tilde{\ell}_t (x^*)}\\
&\overset{(b)}{\le} \frac{L}{\sigma} \sum_{t=1}^T \frac{\norme{\nabla \ell_t(x_t)}^{-2}}{\sum_{\tau = 1}^t \norme{\nabla \ell_\tau(x_\tau)}^{-2}}\\
&\overset{(c)}{\le} \frac{L}{\sigma} \pr{1 + \log \pr{G^2 \sum_{t=1}^T \norme{\nabla \ell_t(x_t)}^{-2}}}
\end{aligned}
\end{equation}
where (a) is from strong convexity, (b) is by applying Lemma~\ref{osc} to the function $\tilde{\ell}_t(\cdot)$, and (c) is due to Lemma~\ref{lem:Log_sum}. 
Thus, we get
\begin{align} \label{eq:exp}
G^2 \sum_{t=1}^T \| \nabla \ell_t(x_t) \|^{-2} \ge \frac{1}{3} e^{\frac{\sigma}{L}T}
\end{align}
\end{proof}

\begin{proof}[Proof of \Cref{thm:linear-rate}]

We consider the weighed loss setting in Algorithm~\ref{alg:game}:
in each round, the
loss function of the x player is $\alpha_t \ell_t(\cdot) : \XX \to \reals$, where $\ell_t(\cdot) := g(\cdot, y_t)$. The $y$-learner, on the other hand, observes her own sequence of loss functions $\alpha_t  h_t(\cdot) : \YY \to \reals$, where $h_t(\cdot) := -  g(x_t, \cdot)$.
To continue, let us
``re-define'' the regret of the $x$-player in the weighted loss game:
$\regret{x}:= \sum_{t=1}^T ( \alpha_t  \ell_t(x_t) - \alpha_t \ell_t(x^*) )$, 
where
$x^* = \arg \min_{x \in \XX} s(x)$.

\begin{eqnarray}
\frac{1}{\sum_{s=1}^T \alpha_s}  \sum_{t=1}^T  \alpha_t g(x_t, y_t) & = & \frac{1}{\sum_{s=1}^T \alpha_s} \sum_{t=1}^T  \alpha_t \ell_t(x_t) \notag \\
  & = &  
\frac{1}{\sum_{s=1}^T \alpha_s} \sum_{t=1}^T  \alpha_t \ell_t(x^*)
+ \frac{1}{\sum_{s=1}^T \alpha_s} \sum_{t=1}^T  \alpha_t ( \ell_t(x_t) - \ell_t(x^*) )
      \notag \\
  & = &  \frac{1}{\sum_{s=1}^T \alpha_s} \sum_{t=1}^T  \alpha_t \ell_t(x^*) + \frac{ \regret{x} }{  \sum_{s=1}^T \alpha_s }  \notag \\
  & \leq &  \sum_{t=1}^T \frac{1}{\sum_{s=1}^T \alpha_s} \alpha_t g(x^*, y_t)  + \epsilon_T \notag \\
  & \overset{(a)}{\leq} & 
    g\left(x^*,{ \textstyle \sum_{t=1}^T \frac{1}{\sum_{s=1}^T \alpha_s} \alpha_t y_t}\right) + \epsilon_T 
    \label{hi2} \\
  & \leq & 
    \max_{y \in \YY} g(x^*,y)  + \epsilon_T    \\ 
  & \overset{(b)}{ = } & 
     s (x^*)  + \epsilon_T    \\ 
  & = & \min_{x \in \XX} \max_{y \in \YY}   g(x,y) + \epsilon_T \notag,
\end{eqnarray}
where $(a)$ is by Jensen, $(b)$ is by the definition of the function $s(x)$.
Then, by combining the above result with (\ref{eq:ylowbound}), one get the following result 
\begin{proposition}
Denote $A_{T}=\sum_{s=1}^T \alpha_s$, $\epsilon_{T}$ the upper bound of 
$\frac{ \regret{x} }{ A_T}$, and $\delta_{T}$ the upper bound of 
$\frac{ \regret{y} }{ A_T}$.  Then,
$\bar{x}_{T} := \frac{ \sum_{s=1}^T \alpha_s x_s  }{ A_T }$ satisfies
$$  \max_{y \in \YY} g( \bar{x}_{T}, y)  \leq  V^* + \epsilon_T + \delta_T.$$
\end{proposition}
However, since the $y$-player plays best response, $\regret{y}=0$, we only need to consider
$\frac{ \regret{x} }{  A_T }$.

Here, we consider the weight to be
$\alpha_t = \norme{\nabla \ell_t(x_t)}^{-2}$. 
\begin{equation}
\begin{aligned}
& \frac{ \regret{x} }{  A_T } :=
\frac{\sum_{t=1}^T \alpha_t(\ell_t (x_t) - \ell_t(x^*))}{A_T} 
\overset{(a)}{\le} \frac{1}{2\sigma A_T} \sum_{t=1}^T \frac{\norme{\nabla \ell_t(x_t)}^{-2}}{\sum_{\tau = 1}^t \norme{\nabla \ell_t(x_\tau)}^{-2}}\\
&\overset{(b)}{\le} \frac{G^2\pr{1 + \log \pr{G^2 \sum_{t=1}^T \norme{\nabla \ell_t(x_t)}^{-2}}}}{2\sigma G^2 \sum_{t=1}^{T} \norme{\nabla \ell_t(x_t)}^{-2}}
\overset{(c)}{\le} \frac{3 G^2\pr{1 + \pr{\frac{\sigma T}{L}}}}{2 \sigma e^{\frac{\sigma}{L}T}} = O\pr{\frac{G^2 T}{L}e^{-\frac{\sigma}{L}T}}
\end{aligned}
\end{equation}
where $(a)$ has be shown in (\ref{tmyt}), $(b)$ is by Lemma~\ref{lem:Log_sum}, $(c)$ is by (\ref{eq:exp}) and the fact that $\frac{1+\log z}{z}$ is monotonically decreasing for $z \ge 1$.

Next, we relate the weighted regret to the game
\begin{equation} \label{hereW}
\begin{aligned}
& \frac{ \regret{x} }{  A_T } :=
\frac{\sum_{t=1}^T \alpha_t(\ell_t (x_t) - \ell_t(x^*))}{A_T} 
= \frac{\sum_{t=1}^T \alpha_t(s(x_t) - g(x^*,y_t))}{A_T} 
\\ &\overset{(a)}{\ge} h(\bar{x}_T) - g(x^*,\bar{y}_T) = g(\bar{x}_T,\yx{\bar{x}_T})  - g(x^*,\bar{y}_T),
\end{aligned}
\end{equation}
where in $(a)$ we use Jensen for both sums.
Thus, our final result is the following,
\begin{equation} \label{f1}
g(\bar{x}_T,\yx{\bar{x}_T})  - g(x^*,\bar{y}_T) \le O\pr{\frac{G^2 T}{L}e^{-\frac{\sigma}{L}T}},
\end{equation}
which also implies
\begin{equation}
\begin{aligned}
& g(\bar{x}_T, \yx{\bar{x}_T}) - \min_{x \in \XX} \max_{y\in \YY} g(x,y) = g(\bar{x}_T, \yx{\bar{x}_T}) - g(x^*,\yx{x^*}) \overset{(a)}{\le} g(\bar{x}_T,\yx{\bar{x}_T}) - g(x^*,\bar{y}_T) 
\\ & \overset{(\ref{f1})}{\le} O\pr{\frac{G^2 T}{L}e^{-\frac{\sigma}{L}T}},
\end{aligned}
\end{equation}
where (a) is because of $g(x^*,\bar{y}_T) \geq g(x^*,\yx{x^*})$ by definition of best response.

But, we also note that $g(\bar{x}_T, \yx{\bar{x}_T}) = s(\bar{x}_T) \geq  \min_{x\in \XX} s(x) = \min_{x \in \XX} \max_{y \in \YY} g(x,y)$. This shows that $g(\bar{x}_T, \yx{\bar{x}_T})$ is an $O(\exp(-T))$ approximate saddle point solution, so we have completed the proof.

\end{proof}

\section{Proof of Theorem~\ref{thm:linearFW}} \label{app:linearFW}

\noindent\textbf{Theorem~\ref{thm:linearFW}}
\textit{ 
Consider the FW game in which $g(x,y)= -x^{\top}y + f^{*}(x)$.
Suppose that $f(\cdot)$ is a $L$-smooth convex function
and that $\YY$ is a $\lambda$-strongly convex set. Also assume that the gradients of the $f$ in $\YY$ are bounded away from $0$, i.e., $\max_{y\in\YY}\|\nabla f(y)\|\geq B$. Let $T \geq \frac {L}{\lambda B} \log 3$.
Then, there exists a FW-like algorithm that has $O(\exp(-\frac{\lambda B }{L} T))$ rate.
}

In this section, we give a proof of \Cref{thm:linearFW}. Let us first restate the theorem with the explicit algorithm.
Note that in the algorithm that we describe below the weights $\alpha_t$ are not predefined but rather depend on the queries of the algorithm. These adaptive weights are explicitly defined in Algorithm~\ref{alg:SC-AFTL} which is used by the $x$-player. The $y$-player plays best response.

\begin{algorithm}[h] 
  \caption{Strongly-Convex Adaptive Follow-the-Leader (SC-AFTL)}
  \label{alg:SC-AFTL}
  \begin{algorithmic}[1]
  \FOR{$t= 1, 2, \dots, T$}
    \STATE Play $x_t \in \K$
    \STATE Receive a strongly convex loss function $\alpha_t \ell_{t}(\cdot)$ with $\alpha_t = \frac{1}{\| \nabla \ell_t(x_t) \|^2} $.
    \STATE Update $x_{t+1} = \min_{x\in \XX} \sum_{s=1}^t \alpha_x \ell_s(x) $  
   \ENDFOR 
  \end{algorithmic}
\end{algorithm}

Before proving the theorem,
let us focus on the strategy of the $x$-player which we describe in Algorithm~\ref{alg:SC-AFTL}.
This algorithm  is equivalent to performing FTL updates  over the following loss sequence:
 $$\left\{\tilde{\ell}_t(x) =\alpha_t \ell_t(x) \right\}_{t=1}^T$$ 
And in the case of FW game, the $L$-smoothness of $f$ implies  that each $\tilde{\ell}_t(x)$ is $\frac{1/L}{\| \nabla \ell_t(x_t)\|^2}$-strongly-convex. 
Thus Lemma~\ref{lemma:ogd_strCvx} implies the following holds for any $x\in\XX$:


\begin{equation} \label{eq:RegretStronglycVX}
\sum_{t=1}^T \tilde{\ell}_t(x_t)-\sum_{t=1}^T \tilde{\ell}_t(x) 
\le 
\frac{L}{2} \sum_{t=1}^T \frac{\| \nabla \ell_t(x_t) \|^{-2}}{\sum_{s=1}^t \| \nabla \ell_s(x_s) \|^{-2}}~.
\end{equation}

\begin{proof}[Proof of \Cref{thm:linearFW}]
Since the $y$-player plays best response, $\regret{y}=0$, we only need to show that
$\regret{x} \leq O(\exp(-\frac{\lambda B}{L} T))$, which we do next.

We start by showing that $s(x) = \max_{y \in \YY} - x^\top y + f^*(x)$ is a smooth function.
We have that 
\begin{equation}
\begin{aligned}
& \| \nabla_w s(\cdot) - \nabla_z s(\cdot) \| = 
\| \arg\max_{y \in \YY} ( - w^\top y + f^*(w) ) -  \arg\max_{y \in \YY} ( - z^\top y + f^*(z) ) \|
\\ &= \| \arg\max_{y \in \YY} ( - w^\top y) - (\arg\max_{y' \in \YY}  - z^\top y') \|
\leq  \frac{ 2 \| w - z \|}{ \lambda ( \| w\| + \| z\| ) }
\leq  \frac{ \| w - z \|}{ \lambda B},
\end{aligned}
\end{equation}
where the second to last inequality uses Lemma~\ref{lm:lip} regarding $\lambda$-strongly convex sets,
and the last inequality is by assuming the gradient of $\| \nabla f(\cdot)\| \geq B$
and the fact that $w,z \in \XX$ are gradients of $f$.
This shows that $s(\cdot)$ is a smooth function with smoothness constant $L'=\frac{1}{\lambda B}$.

\begin{align}
T &= \sum_{t=1}^T \frac{\norme{ \nabla \ell_t(x_t)}^2}{\norme{ \nabla \ell_t(x_t)}^2}
\overset{Proposition~\ref{sameGrad}}{ = }\sum_{t=1}^T \frac{\norme{\nabla s(x_t)}^2}{\norme{ \nabla \ell_t(x_t)}^2}
\overset{Lemma~\ref{lem:GSmooth}}{ \le} \sum_{t=1}^T \frac{L'}{\norme{ \nabla \ell_t(x_t)}^2} (s(x_t) - s(x^*))\\
&\le \sum_{t=1}^T \frac{L'}{\norme{ \nabla \ell_t(x_t)}^2} (\ell_t(x_t) - \ell_t(x^*))\label{eq:normal},
\end{align}
where the last inequality follows from the fact that $s(x_t) = \ell_t(x_t)$ and $\ell_t(x) = g(x,y_t) \le g(x,\yx{x}) = s(x)$ for any $x$. We can apply Lemma~\ref{lem:GSmooth} because 
$\arg\min_{x \in \reals^d} \{ \max_{y \in \YY} - x^\top y + f^*(x) \} \in \XX$, as $\XX$ is the gradient space. 

To continue, we have
\begin{equation} \label{eq:ExpRate}
\begin{aligned}
T & \leq
\sum_{t=1}^T \frac{{ L' }}{\| \ell_t( x_t) \|^2} ( \ell_t(x_t)-\ell_t(x^*) )  \nonumber \\
&\overset{(a)}{=}
\sum_{t=1}^T L' ( \tilde{\ell}_t(x_t)-\tilde{\ell}_t(x^*) )  \nonumber \\
&\overset{(b)}{\le}
\frac{L\cdot L'}{2} \sum_{t=1}^T \frac{\| \nabla \ell_t( x_t) \|^{-2}}{\sum_{s=1}^t \|\nabla \ell_s(x_t) \|^{-2}}\nonumber \\
&\overset{(c)}{\le}
\frac{L\cdot L'}{2}\left( 1+\log(G^2 \sum_{t=1}^T\|\nabla \ell_t( x_t)\|^{-2}) \right)~,
\end{aligned}
\end{equation}
where (a) is by the definition of $\tilde{\ell}_t(\cdot)$,
and (b) is shown using (\ref{eq:RegretStronglycVX}) with strong convexity parameter of $\ell_t(\cdot)$ being $\frac{1}{L}$, and (c) is by Lemma~\ref{lem:Log_sum},
\begin{align*}
\sum_{t=1}^T \frac{\| \ell_t( x_t) \|^{-2}}{\sum_{s=1}^t \| \ell_s( x_s) \|^{-2}}
&=
\sum_{t=1}^T \frac{G^2\| \ell_t( x_t) \|^{-2}}{\sum_{s=1}^t G^2\| \ell_s( x_s) \|^{-2}}
\leq{}
{1+\log(G^2 \sum_{t=1}^T\| \ell_t( x_t)\|^{-2})}.
\end{align*}
Thus, we get
\begin{align} 
G^2 \sum_{t=1}^T \| \nabla \ell_t(x_t) \|^{-2} = O(  e^{\frac{1}{L\cdot L'}T}) =  O(  e^{\frac{\lambda B}{L}T}).
\end{align}
The rest of the proof is the same as in (\ref{hereW}). 
So, we can get that $g(\bar{x}_T, \yx{\bar{x}_T})$ is an $O(\exp(-\frac{\lambda B}{L}T))$ approximate saddle point solution. This completes the proof.

\end{proof}

\section{Supremum of strongly convex functions is strongly convex}

\begin{lemma} \label{mysc}
$s(x) = \underset{y \in \YY}{\sup}\; g(x,y)$ is $\sigma$-strongly convex if $g(\cdot,y)$ is $\sigma$-strongly convex for any $y \in \YY$.
\end{lemma}

\begin{proof}
For any $\theta=[0,1]$ and any $w,z \in X$,
\begin{equation}
\begin{aligned}
& s(\theta w + (1-\theta) z ) = g( \theta w + (1-\theta) z , \hat{y} )
\\ &\leq \theta g(w, \hat{y}) +  (1-\theta) g(z, \hat{y}) - \frac{\sigma}{2} \theta (1 -\theta) \| w -z \|^2
\\ &\leq \theta s(w) + (1-\theta) s(z) - \frac{\sigma}{2} \theta (1 -\theta) \| w -z \|^2,
\end{aligned}
\end{equation}
where the first equality holds for a specific $\hat{y} \in \YY$.

\end{proof}

\end{document}